\documentclass[final,onefignum,onetabnum]{siamart171218}

\usepackage{lipsum}
\usepackage{amsfonts}
\usepackage{graphicx}
\usepackage{epstopdf}
\usepackage{amsmath}
\usepackage{amssymb}
\usepackage{calligra}
\usepackage{calrsfs}
\usepackage{algorithm}
\usepackage{algorithmicx}
\usepackage{algpseudocode}
\usepackage{booktabs}
\usepackage{subfig}
\usepackage{epsfig}
\usepackage{color}
\usepackage{enumerate}

\def\dist{\rm dist}

\ifpdf
  \DeclareGraphicsExtensions{.eps,.pdf,.png,.jpg}
\else
  \DeclareGraphicsExtensions{.eps}
\fi

\newsiamremark{remark}{Remark}
\newsiamremark{hypothesis}{Hypothesis}
\crefname{hypothesis}{Hypothesis}{Hypotheses}
\newsiamthm{claim}{Claim}

\headers{A Level Set Representation Method 
for Convex Shape \& Applications}{L. Li, S. Luo, XC. Tai and J. Yang}
\title{A  Level Set Representation Method 
for N-dimensional Convex Shape and Applications}
\author{Lingfeng Li\thanks{Department of Mathematics, Hong Kong Baptist University, Hong Kong, China; Department of Mathematics, Southern University of Science and Technology, Shenzhen, China}
\and Shousheng Luo\thanks{
School of Mathematics and Statistics, Data Analysis Technology Lab, Henan University, Kaifeng, China}
\and Xue-Cheng Tai\thanks{Department of Mathematics, Hong Kong Baptist University, Hong Kong, China (\email{xuechengtai@hkbu.edu.hk})} 
\and Jiang Yang\thanks{Department of Mathematics, Southern University of Science and Technology, Shenzhen, China}
}

\usepackage{amsopn}

\makeatletter
\newcommand*{\addFileDependency}[1]{% argument=file name and extension
  \typeout{(#1)}% latexmk will find this if $recorder=0 (however, in that case, it will ignore #1 if it is a .aux or .pdf file etc and it exists! if it doesn't exist, it will appear in the list of dependents regardless)
  \@addtofilelist{#1}% if you want it to appear in \listfiles, not really necessary and latexmk doesn't use this
  \IfFileExists{#1}{}{\typeout{No file #1.}}% latexmk will find this message if #1 doesn't exist (yet)
}
\makeatother

\newcommand{\R}{\mathbb{R}}

\newcommand{\dom}{\textbf{dom}}
\ifpdf
\hypersetup{
  pdftitle={A Level-Set Method for n-D Convex Shape \& Applications},
  pdfauthor={L. Li, S. Luo, XC. Tai and J. Yang}
}
\fi

\usepackage{placeins}

\begin{document}
\maketitle

% REQUIRED
\begin{abstract}
In this work, we present a new efficient method for convex shape representation, which is regardless of the dimension of the concerned objects, using level-set approaches. Convexity prior is very useful for object completion in computer vision. It is a very challenging task to design an efficient method for high dimensional convex objects representation. In this paper, we prove that the convexity of the considered object is equivalent to the convexity of the associated signed distance function. Then, the second order condition of convex functions is used to characterize the shape convexity equivalently. We apply this new method to two applications: object segmentation with convexity prior and convex hull problem (especially with outliers). For both applications, the involved problems can be written as a general optimization problem with three constraints. Efficient algorithm based on alternating direction method of multipliers is presented for the optimization problem. Numerical experiments are conducted to verify the effectiveness and efficiency of the proposed representation method and algorithm.

%In this paper we present a new continuous method for convexity shape representation via level set function, which 
%is regardless of the dimension of the concerned objects. We prove that the equivalent condition  

% When using the signed distance function to represent a shape, we prove that the convexity of the shape is equivalent to the convexity of the distance function. Based on this observation, we construct our new prior using the second order condition of a convex function. Compared to other methods, the main advantage of this prior is that it works not only for 2 dimension but also higher dimensions efficiently. To test our new prior, we apply it to two different types of models, which are segmentation models and convex hull models, and design appropriate algorithm. The algorithm is based on the Alternating direction methods of multipliers and the convexity constraint can be handled easily by a simple projection. We then conduct numerical experiments for these two models on 2-D and 3-D cases. The results shows that our convexity prior can characterize the convex shapes very accurately.
\end{abstract}

% REQUIRED

% REQUIRED

\section{Introduction}
In the tasks of computer vision, especially image segmentation, shape priors are very useful information to improve output results when the objects of interest are partially occluded or suffered from strong noises, intensity bias and artifacts. Therefore, various shape priors are investigated in the literature \cite{chan2001active,gorelick2014convexity,leventon2002statistical,Yuan2013}. In \cite{chen2002using} and \cite{leventon2002statistical}, the authors combined shape priors with the snakes model \cite{caselles1997geodesic} using a statistical approach and a variational approach, respectively. Later, based on the Chan-Vese model \cite{chan2001active}, a new variational model, which uses a labelling function to deal with the shape prior, was proposed in \cite{cremers2003towards}. A modification of this method was presented in \cite{chan2005level} to handle the scaling and rotation of the prior shape. All the priors used in these papers are usually learned or obtained from some given image sets specifically.

Recently, generic and abstract shape priors have attracted more and more attentions, such as connectivity \cite{vicente2008graph}, star shape \cite{veksler2008star,yuan2012efficient}, hedgehog \cite{isack2016hedgehog} and convexity \cite{Gorelick2017Convexity,Yuan2013}. Among them, the convexity prior is one of the most important priors. Firstly, many objects in natural and biomedical images are convex, such as balls, buildings and some organs \cite{royer2016convexity}. Secondly, convexity also plays a very important role in many computer vision tasks, like human vision completion \cite{Liu1999The}. Several methods for convexity prior representation and its applications were discussed in the literature \cite{gorelick2014convexity,Yuan2013,yangA2017}. However, these methods often work for 2-dimensional convex objects only and may have relatively high computational costs. In this paper, we will present a new method for convexity shape representations. This method is not only suitable for convex objects in all dimensions but also numerically efficient for computations.

Most of the existing methods for convex shape representation can be divided into two groups: discrete approaches and continuous approaches. For the first class, there are several methods in the literature. In \cite{strekalovskiy2011generalized}, the authors first introduced a generalized ordering constraint for convex object representation. To achieve the convexity of objects, one needs to explicitly model the boundaries of objects. Later, an image segmentation model with the convexity prior was presented in \cite{gorelick2014convexity}. This method is based on the convexity definition and the key idea is penalizing all 1-0-1 configurations on all straight lines where 1 (resp. 0) represents the associated pixel inside (resp. outside) the considered object. This method was then generalized for multiple convex objects segmentation in \cite{gorelick2017multi}. In \cite{royer2016convexity}, the authors proposed a segmentation model which can handle multiple convex objects. They formulated the problem as a minimum cost multi-cut problem with a novel convexity constraint: If a path is inside the concerned object(s), then the line segment between the two ends should not pass through the object boundary.

The continuous methods usually characterize the shape convexity by exploiting the non-negativity of the boundary curvature in 2-dimensional space. As far as we know, the curvature non-negativity was firstly incorporated into image segmentation in \cite{Yuan2013}. Then, a similar method was adopted for cardiac left ventricle segmentation in \cite{yangA2017}. In \cite{bae2017augmented}, an Euler's elastica energy-based model was studied for convex contours by penalizing the integral of absolute curvatures. 
%Generally speaking, these methods only use the curvature information of the object boundary. In addition, the curvature of object boundary is computed by the curvature formula \cite{Goldman2005Curv} of corresponding level-set functions.
Recently, the continuous methods or curvature-based methods were developed further in \cite{luo2018convex,yan2018convexity}. For a given object in 2-dimensional space,
it was proved in \cite{yan2018convexity} that the non-negative Laplacian of the associated signed distance function \cite{sussman1994level} (SDF) is sufficient to guarantee the convexity of shapes.In \cite{luo2018convex}, the authors also proved that this condition is also necessary. 
%The proof was derived from a new observation that the Laplacian of a SDF is just the curvatures of all level-set curves. 
In addition, instead of solving negative curvature penalizing problems like \cite{bae2017augmented,Yuan2013,yangA2017}, these two papers incorporated non-negative Laplacian condition as a constraint into the involved optimization problem. This method has also been extended for multiple convex objects segmentation in \cite{luo2019convex}. Projection algorithm and alternating direction method of multipliers (ADMM) were presented in \cite{yan2018convexity} and \cite{luo2018convex}.
% respectively. Later, the authors of \cite{li2019convex} also applied this convexity representation method to the convex hull problems.

In some real applications, one needs to preserve the convexity for high dimensional objects, such as tumors and organs in 3D medical images. Therefore, it is very significant to study efficient representation methods for high dimensional convex objects. However, there are some difficulties in theories and numerical computation for the existing methods mentioned above to be generalized to higher dimensions.

For the discrete methods in \cite{Gorelick2017Convexity,royer2016convexity}, it is almost impossible to extend these methods to higher dimensional ($\geq 3$) cases directly because of the computational complexity. One can setup the same model as the two dimensional case, but the computational cost will increase dramatically. For the continuous methods using level-set functions \cite{luo2018convex,Yuan2013}, the mean curvature (Laplacian of the SDF) of the zero level-set surface can not guarantee the convexity of the object convexity in high dimensions. For example, in the 3-dimensional space, an object is convex if and only if its two principle curvatures always have the same sign at the boundary. Accordingly, one should use the non-negativity of Gaussian curvatures to characterize the convexity of objects \cite{elsey2009analogue}. However, the problems (e.g., optimization problem for image segmentation) with non-negative Gaussian curvature constraint or penalty are very complex and difficult to solve.

In this work, we present a new convexity prior which works in any dimension. Similar to the framework in \cite{luo2018convex}, we also adopt the level-set representations, which is a powerful tool for numerical analysis of surfaces and shapes \cite{osher2002level,peng1999pde,zhao1996variational} and have many applications in image processing \cite{chan1999active,vese2002multiphase}. We first prove the equivalence between the object convexity and the convexity of the associated SDF. It is well-known that a convex function must satisfy the second order conditions, i.e., having positive semi-definite Hessian matrices at all points if it is secondly differentiable. Based on this observation, we obtain a new way to characterize the shape convexity using the associated SDF. The proposed methods have several advantages. Firstly, it works regardless of the object dimensions. Secondly, this method can be easily extended to multiple convex objects representation using the idea in \cite{luo2019convex}. Thirdly, this representation is very simple and allows us to design efficient algorithms for potential applications.

To verify the effectiveness of the proposed method, we apply it to two types of important applications in computer vision and design an efficient algorithm to solve them. The first one is the image segmentation task with convexity prior  \cite{Gorelick2017Convexity,luo2018convex,Yuan2013}. More specifically, we combine the convexity representation with a 2-phase probability-based segmentation model. 
%The model is given as an optimization problem with a positive semi-definite Hessian constraint, where the objective functional contains a data fidelity term and a boundary regularization term. This kind of model have been widely studied and have many fast algorithm \cite{chambolle2011first,yuan2010continuous}. There is also no difficulty to generalize it to multi-phase problems.% or high dimensional problems.
The second model is the variational convex hull problem,
%Seeking the convex hull of a given set is a very important task in computational geometry. Various efficient algorithms have been proposed \cite{BarberC1996Tqaf,graham1972,jarvis1973}. However, most of the existing methods can not handle the data containing intensive outliers very well. Recently, 
which was first introduced in \cite{li2019convex} for 2-dimensional binary images. 
%the authors introduced a variational convex hull model for 2-dimensional binary images. 
This model can compute multiple convex hulls of separate objects simultaneously and is very robust to noises and outliers compared to traditional methods. However, since it uses the same convexity prior with \cite{luo2018convex}, it works only for 2-dimensional problems. Using the proposed method, we can generalize it to higher dimensions and maintain all of its advantages, e.g., robustness to outliers.

Both applications can be formulated as a general constrained optimization problem. 
%with three constraints: positive semi-definite Hessian, unitary gradient and consistency with some prior information, which will be demonstrated in details in Section \ref{sec:algo}. 
Similar to \cite{li2019convex,luo2018convex}, alternating direction method of multiplier (ADMM) for the general optimization problem is derived to solve the models. In the proposed algorithm, the solutions of all sub-problems can be derived explicitly or computed efficiently. 
%In order to save the computational cost, especially for 3-dimensional problems, we only require the Hessian matrices to be positive semi-definite on a belt around the zero level-set of the associated SDF instead of the whole image or volume. In addition, we assume that the consider data or image is defined on a torus, which will result in the periodic boundary condition of the associated SDFs. Consequently, the update of the SDFs can be implemented by FFT efficiently. 
Numerical results for the two concerned problems are presented to show the effectiveness and efficiency of our methods in 2 and 3 dimensional cases. 
% {\red As a by-products of this technique, we can have a very fast way to compute the SFDs in addition to other well-known methods such as Fast Marching or Fast Sweeping \cite{zhao2005fast}}.

The rest of this paper is organized as follow. In Section \ref{sec:pre}, we will introduce some basic notations and results from convex optimization. Then we will present the convexity prior representation method in Section \ref{sec:prior}. Section \ref{sec:model} is devoted to two variational models for the two applications. The numerical algorithm will then be given in Section \ref{sec:algo}. In Section \ref{sec:exp}, we will present some experimental results in 2 and 3 dimensional spaces to test our proposed methods. Conclusions and future works will be discussed in Section \ref{sec:cln}. 

\section{Preliminaries} \label{sec:pre}
First, we will briefly introduce some results from convex optimization in this section. Given a set of points $\{x_i\}_{i=1}^k$ in $\mathbb{R}^d$, a point in the form $\sum_{i=1}^k\theta_i x_i$ is called a \textit{convex combination} of $\{x_i\}$ if $\theta_i\geq 0$ and $\sum_{i=1}^k\theta_i=1$. Then, given a set $C$ in $\mathbb{R}^d$, it is said to be \textit{convex} if all convex combinations of the points in $C$ also belong to $C$. The convex hull or convex envelope of $C$ is defined as the collection of all its convex combinations, i.e., 
\begin{equation}
    \textbf{Conv}(C)=\{\sum_{i=1}^k\theta_ix_i|x_i\in C,\ \theta_i\geq 0 \text{ for }i=1,\dots,k \text{ and } \sum_{i=1}^k\theta_i=1\},
\end{equation}
where $k$ can be any finite positive integer. In other words, $\textbf{Conv}(C)$ is the smallest convex set containg $C$.
% Equivalently, we can also define the convex hull in another way:
% \begin{lemma}{\cite{boyd2004convex}} \label{lemma:smallest_convex_set}
% Given a set $C\subseteq\mathbb{R}^d$, its convex hull is the smallest convex set containing $C$, i.e., for any convex set $C_1\supseteq C$, we have $\textbf{Conv}(C)\subseteq C_1$.
% \end{lemma}

Another important concept is the convex function. Given a function $f(x):\R^d\rightarrow\R\bigcup\{+\infty\}$, $f(x)$ is a \textit{convex function} if its epigraph $\{(x,f(x))|x\in\mathbb{R}^d\}$ is a conves set. If $f$ is twice differentiable, a necessary and sufficient condition for the convexity is that the Hessian matrix of $f$ is positive semi-definite at every $x$, i.e. 
\begin{equation} 
  Hf(x)\geq 0,\label{eq:second_order}
\end{equation} 
where $Hf(x)$ denotes the Hessian matrix of $f$ at $x$. 

For any measurable $f$ and real number $\alpha$, we can define the $\alpha$ level-set, $\alpha$ sublevel-set and $\alpha$ suplevel-set of $f$ as follow:
\begin{align}
    &L_{\alpha}(f)=\{x|x\in\dom(f),\ f(x)=\alpha\},\\ 
    &L_{\alpha}^-(f)=\{x|x\in\dom(f),\ f(x)<\alpha\},\\
    &L_{\alpha}^+(f)=\{x|x\in\dom(f),\ f(x)>\alpha\},
\end{align}
where $\dom(f)$ denotes the domain of $f$. It is easy to verify that if $f(x)$ is convex, $L^-_{\alpha}(f)$ is 
also convex for any $\alpha\in\R$.

Next, we are going to introduce a powerful tool, named level-set function, for the implicit representation of shapes. Given an open set $\Omega_1$ in $\R^d$ with piecewise smooth boundary $\Gamma$, the level-set function, denoted as $\phi(x):\mathbb{R}^d\rightarrow\mathbb{R}$, of $\Gamma$ satisfies:
\begin{equation}
    \begin{cases} \phi(x)>0, & x\in \Omega_0,\\
    \phi(x)=0, & x\in\Gamma,\\
    \phi(x)<0, &x\in \Omega_1,
    \end{cases}
\end{equation}
where $\Omega_0$ is the exterior of $\Gamma$. We further assume that $\Omega_1$ is nonempty and $\Gamma$ has measure zero in the rest of this paper. Equivalently, we can also denote $\Omega_1$ as $L_0^-(\phi)$, $\Gamma$ as $L_0(\phi)$ and $\Omega_0$ as $L_0^+(\phi)$. Then, the evolution of the hypersurface $\Gamma$ can be represented by the evolution of $\phi$ implicitly. The main advantage of the level-set method is that it can track complicated topological changes and represent sharp corners very easily. Using the heaviside function $h(\cdot)$, we can obtain the characteristic function of $\Omega_0$ by $h(\phi(x))$ and the characteristic function of $\overline{\Omega}_1$ by $1-h(\phi(x))$. The distributional derivative of the heaviside function is denoted as $\delta$.

In this work, we are going to use the signed distance function, which is a special type of level-set function, to represent convex shapes. For an open set $\Omega_1$ with piecewise smooth boundary $\Gamma$, the SDF of $\Gamma$ is defined as
\begin{equation}
    \phi(x)=\begin{cases} \dist(x,\Gamma), & x\in\Omega_0, \\
    0, & x\in\Gamma,\\
    -\dist(x,\Gamma), & x\in\Omega_1,
    \end{cases}
\end{equation}
where $\dist(x,\Gamma)=\underset{z\in\Gamma}{\text{inf}}\Vert z-x\Vert_2$. Since $\Gamma$ is a compact set, the infimum can be obtained in $\Gamma$. It is well-known that the SDF is continuous and satisfies the Eikonal equation \cite{osher2002level}:
\begin{equation}
    |\nabla\phi|=1, \label{eq:eikonal}
\end{equation}
where the gradient is defined in the weak sense. Notice that the weak solution of (\ref{eq:eikonal}) is not unique, but one can define a unique solution in the viscosity sense \cite{crandall1992user} given certain boundary conditions. We also obtained an interesting property for the SDF:
\begin{lemma} \label{lemma:smallest_SDF}
Suppose we are given two compact subset $C_1$ and $C_2$ in $\R^d$. Denotes their boundaries by $\Gamma_1$ and $\Gamma_2$, and their SDFs by $\phi_1$ and $\phi_2$ respectively. Then $C_1\subseteq C_2$ if and only if $\phi_1(x)\geq\phi_2(x)$ for any $x$.
\end{lemma}
The proof of this result will be given in the next section. This result is very useful in computing the convex hull via level-set representation.

In \cite{luo2018convex}, the authors presented an equivalent condition of 2-dimensional object convexity: 
\begin{equation}
    \Delta\phi\geq 0\label{eq:prior_1},
\end{equation}
where $\phi$ is the associated SDF. It can be shown that $\Delta\phi(x)$ equals to the mean curvature of the level-set curve $L_{c}(\phi)$ 
at the point $x$, where $c=\phi(x)$. In the 2 dimensional space, if the mean curvatures of the zero level-set curves are non-negative, the object $\Omega_1$ must be convex. However, this is not the case in higher dimensions.

\section{Convexity representation for high dimensional shapes} \label{sec:prior}
In \cite{luo2018convex}, the authors proved that a 2-dimensional shape is convex if and only if any sublevel-set of its SDF is convex. Actually, one can show that the convexity of a shape is also equivalent to the convexity of its SDF and this is true in any dimensions. 
We summarize this result in the following theorem. 

\begin{theorem}\label{thm:equivalence}
Let $\Gamma$ be the boundary of a bounded open convex subset $\Omega_1\subset\R^d$ and $\phi$ be the corresponding SDF of $\Gamma$. Then, $\Omega_1$ is convex if and only if $\phi(x)$ is convex.
\end{theorem}
It is well-known that the SDF $\phi$ must satisfy the second order condition (\ref{eq:second_order}) if it is secondly differentiable. Therefore, we can use the condition (\ref{eq:second_order}) to represent the convexity of shapes. Before proving Theorem \ref{thm:equivalence}, we need to introduce some useful lemmas.

% \begin{lemma}[Projection Theorem \cite{ciarlet2013linear}] ??This lemma is well-known, no need to list it here??
% \label{lemma:projection_theorem}
% Let $Z$ be a nonempty, convex and complete subset of a real inner-product space $(X,\langle\cdot,\cdot\rangle)$. Let $\Vert\cdot\Vert$ be the norm introduced by the inner-product. Then Given any $x\in X$, there exist a unique element $Px\in Z$ such that 
% \begin{align}
%     &\Vert x-Px\Vert=\underset{z\in Z}{\inf}\Vert z-x\Vert,\\
%     &\langle x-Px,z-Px\rangle\leq 0 \text{ for any } z\in Z.
% \end{align}
% This $Px$ is denoted as the projection of $x$ onto $Z$. For any two elements $x$ and $y$ in $X$, we have
% \begin{equation}
%     \Vert Px-Py\Vert\leq\Vert x-y\Vert.
% \end{equation}
% \end{lemma}

% Following from the projection theorem, we can also derive another important result.

\begin{lemma} \label{lamma:project_to_surface}
Let $\Gamma$ be the boundary of a bounded open convex subset $\Omega_1\subset\R^d$ and $\phi$ be the corresponding SDF of $\Gamma$. For any $x$ in $\overline{\Omega}_0$, there exists a unique $Px\in\Gamma$ such that 
\begin{equation}
    \Vert x-Px\Vert_2=\underset{z\in \overline{\Omega}_1}{\inf}\Vert x-z\Vert_2=\underset{z\in\Gamma}{\inf}\Vert x-z\Vert_2=\dist(x,\Gamma)=\phi(x).
\end{equation}
In other words, the projection of an exterior point must belong to $\Gamma$.
\end{lemma}
\begin{proof}
Since $\overline{\Omega}_1$ is complete, for any $x\in\overline{\Omega}_0$, there exists an element $Px\in\overline{\Omega}_1$ such that $\Vert x-Px\Vert_2=\underset{z\in\overline{\Omega}_1}{\inf}\Vert z-x\Vert_2$. If $x\in\Gamma$,  then $x=Px$ and the proof is trivial. If $x\notin\Gamma$, we have $x-Px\neq 0$. If $Px\in\Omega_1$ which is open, there exists an open ball $B(Px,r)$ centered at $Px$ with radius $r>0$ such that $B(Px,r)\subseteq \Omega_1$. Let $z=\frac{r(x-Px)}{2\Vert x-Px\Vert_2}+Px$. We can verify that $z\in B(Px,r)\subseteq \Omega_1$ and $\langle x-Px,z-Px\rangle=\frac{r}{2}\Vert x-Px\Vert_2>0$, which contradicts to the projection theorem \cite[Theorem 4.3-1]{ciarlet2013linear}. Therefore, $Px$ must belong to $\Gamma$ and $\Vert x-Px\Vert_2=\phi(x)$. What's more, by the projection theorem, this $Px$ is unique.
\end{proof}

Actually, one can generalize this result to non-convex $\Omega_1$ without difficulty. The next lemma can be directly derived from the definition of SDF.

\begin{lemma} \label{lemma:deviate_from_surface}
Let $\Gamma$ be the boundary of an open subset $\Omega_1\subset\R^d$ and $\phi$ be the corresponding SDF of $\Gamma$. Then for any element $x\in\overline{\Omega}_1$ and non-negative $c$, the inequality $\phi(x)\leq -c$ is true  if and only if $\overline{B(x,c)}\subseteq \overline{\Omega}_1$, where $B(x,c)=\{z\in\R^d|\Vert z-x\Vert_2<c\}$.
\end{lemma}
% \begin{proof}
% Suppose $\phi(x)\leq -c$, which implies that $\underset{z\in\Gamma}{\inf} \Vert z-x\Vert\geq c$. If $x\in\Gamma$, then $\overline{B(x,0)}=\{x\}\subseteq\overline{\Omega}_1$. If $x\in\Omega_1$ and there exist a $y\in \overline{B(x,c)}$ such that $\phi(y)>0$, by the continuity of $\phi$ and convexity of $B(x,c)$, there exist a $z\in B(x,c)$ such that $z\in\Gamma$ and $\Vert z-x\Vert<c$, which contradicts to the assumption.

% Conversely, suppose $B(x,c)\subseteq\overline{\Omega}_1$. If $x\in\Gamma$, the proof is trivial. If $x\in\Omega_1$ and there exist a $y\in\Gamma$ such that $\Vert y-x\Vert<c$, then we have $y\in B(x,c)$. Since $y$ is an interior point of $B(x,c)$, it is also an interior point of $\Omega_1$, which is a contradiction.
% \end{proof}
A simple corollary of the Lemma \ref{lemma:deviate_from_surface} is that $\overline{B(x,|\phi(x)|)}\subseteq \overline{\Omega}_1$ for any $x\in \overline{\Omega}_1$. Now, we can give the proof of Theorem \ref{thm:equivalence} as follows:
\begin{proof}
We first assume $\Omega_1$ is convex. Let $x_1$ and $x_2$ be any two elements in $\R^d$ and $x_0=\theta x_1+(1-\theta)x_2$. We would like to show that $\phi(x_0)\leq\theta\phi(x_1)+(1-\theta)\phi(x_2)$ for any $\theta\in[0,1]$. We will divide the proof into three parts. 

(i) First, if $x_1$ and $x_2$ are in $\overline{\Omega}_0$, i.e., $\phi(x_1)\geq 0$ and $\phi(x_2)\geq 0$. By Lemma \ref{lamma:project_to_surface}, there exist unique $y_1\in\Gamma$ and $y_2\in\Gamma$ such that $\phi(x_1)=\Vert x_1-y_1\Vert_2$ and $\phi(x_2)=\Vert x_2-y_2\Vert_2$. If $x_0\in \Omega_0$, let $y_0=\theta y_1+(1-\theta)y_2\in \overline{\Omega}_1$, and we have
\begin{align}
    \phi(x_0)&\leq \Vert x_0-y_0\Vert_2=\Vert\theta(x_1-y_1)+(1-\theta)(x_2-y_2)\Vert_2\\
    &\leq\theta\Vert x_1-y_1\Vert_2+(1-\theta)\Vert x_2-y_2\Vert_2\\
    % &=\sqrt{\theta^2\phi(x_1)^2+(1-\theta)^2\phi(x_2)^2+2\theta(1-\theta)\langle x_1-y_1,x_2-y_2\rangle}\\
    % &\leq \sqrt{\theta^2\phi(x_1)^2+(1-\theta)^2\phi(x_2)^2+2\theta(1-\theta)\phi(x_1)\phi(x_2)}\\
    &=\theta\phi(x_1)+(1-\theta)\phi(x_2).
\end{align}
If  $x_0\in\overline{\Omega}_1$, we have $\phi(x_0)\leq0\leq\Vert x_0-y_0\Vert_2$. 
%If $x_0\in\Omega_0$, the first inequality holds because $\phi(x_0)\leq \Vert x_0-z\Vert_2$ for any $z$ in $\overline{\Omega}_1$.

(ii) If $x_1$ and $x_2$ are in $\overline{\Omega}_1$, then $x_0$ is also in $\overline{\Omega}_1$.  Let $r=\theta|\phi(x_1)|+(1-\theta)|\phi(x_2)|$.  For any element $y\in \overline{B(x_0,r)}$, there exist $v_1$ and $v_2$ such that $y=x_0+\theta v_1+(1-\theta) v_2$, where $\Vert v_1\Vert_2\leq |\phi(x_1)|$ and $\Vert v_2\Vert_2\leq |\phi(x_2)|$. For example, one can take $v_1=\frac{|\phi(x_1)|}{r}$%{\theta|\phi(x_1)|+(1-\theta)|\phi(x_2)|}(y-x_0)$ 
and $v_2=\frac{|\phi(x_2)|}{r}$.%{\theta|\phi(x_1)|+(1-\theta)|\phi(x_2)|}(y-x_0)$. 
Then, we can rewrite $y$ as $y=\theta(x_1+v_1)+(1-\theta)(x_2+v_2).$ Since $x_1+v_1\in \overline{B(x_1,|\phi(x_1)|)}\subseteq \overline{\Omega}_1$ and $x_2+v_2\in \overline{B(x_2,|\phi(x_2)|)}\subseteq\overline{\Omega}_1$, we have $y\in \overline{\Omega}_1$. Thus, $\overline{B(x_0,|r|)}\subseteq\overline{\Omega}_1$ and $\phi(x_0)\leq -r=\theta\phi(x_1)+(1-\theta)\phi(x_2)$.

(iii) If $x_1\in \Omega_1$ and $x_2\in\Omega_0$, there exist a $Px_2\in\Gamma$ such that $\phi(x_2)=\Vert x_2-Px_2\Vert_2$. By the continuity of $\phi$, there exist a $x_3=\alpha x_1+(1-\alpha)x_2$ where $\alpha\in(0,1)$, $\phi(x_3)=0$, and $x_0\in\Omega_0$ for any $\theta<\alpha$. Denote $y_0=\alpha x_1+(1-\alpha)Px_2\in\overline{\Omega}_1$. We can compute that
\begin{equation}
    \Vert x_3-y_0\Vert_2=\Vert(1-\alpha)(x_2-Px_2)\Vert_2=(1-\alpha)\phi(x_2).
\end{equation}
By (ii), we have 
\begin{equation}-(1-\alpha)\phi(x_2)=-\Vert x_3-y_0\Vert_2\leq\phi(y_0)\leq\alpha\phi(x_1)+(1-\alpha)\phi(Px_2)=\alpha\phi(x_1).\end{equation}
Consequently, we have $0=\phi(x_3)\leq \alpha\phi(x_1)+(1-\alpha)\phi(x_2)$. For any $1\geq\theta>\alpha$, $x_0=\theta x_1+(1-\theta)x_2\in \overline{\Omega}_1$, and then $x_0=\frac{1-\theta}{1-\alpha}x_3+(1-\frac{1-\theta}{1-\alpha})x_1$. Since $x_0$ and $x_1$ are in $\overline{\Omega}_1$, we know that
\begin{align}
\phi(x_0)&\leq \frac{1-\theta}{1-\alpha}\phi(x_3)+(1-\frac{1-\theta}{1-\alpha})\phi(x_1)\\
&\leq \frac{1-\theta}{1-\alpha}(\alpha\phi(x_1)+(1-\alpha)\phi(x_2))+(1-\frac{1-\theta}{1-\alpha})\phi(x_1)\\
&=\theta\phi(x_1)+(1-\theta)\phi(x_2).\end{align}
If $0\leq\theta<\alpha$, we can similarly derive that
\begin{equation}
    \phi(x_0)\leq \frac{\theta}{\alpha}\phi(x_3)+(1-\frac{\theta}{\alpha}\phi(x_2))\leq \theta\phi(x_1)+(1-\theta)\phi(x_2).
\end{equation}

Based on (i), (ii) and (iii), we can conclude that $\phi(x_0)\leq\theta\phi(x_1)+(1-\theta)\phi(x_2)$ for any $x_1$ and $x_2$ in $\R^d$.

Conversely, If $\phi(x)$ is convex in $\R^d$, by the definition of convex function, all the sublevel-sets of $\phi$ are convex, so is $L_0^-(\phi)=\Omega_1$.
\end{proof}

We have already proved that for any convex shape $\Omega_1$, its corresponding SDF $\phi$ must be a convex function. Consequently, $\phi$ must satisfy the second order condition where it is secondly differentiable. Note that a SDF usually is non-differentiable at a set of points with zero measure, so the second order condition holds almost everywhere in the continuous case. In the numerical computation, since we only care about the convexity of $\Gamma$, to save the computational cost, we can only require the condition holds in a neighbourhood around the object boundary, i.e.,
\begin{equation}
    H(\phi)\geq 0 \text{ in } L_{\epsilon}^-(|\phi|)=\{x\in\mathbb{R}^d||\phi(x)|\leq\epsilon\},
\end{equation}
for some $\epsilon>0$.

Lastly, using the Lemma \ref{lamma:project_to_surface} and Lemma \ref{lemma:deviate_from_surface}, we can prove Lemma \ref{lemma:smallest_SDF} as follows:
\begin{proof}
Suppose $C_1\subseteq C_2$, then we would like to show that $\phi_1(x)\geq\phi_2(x)$ for any $x\in\Omega$. If $x\in\Omega\backslash C_2$, by Lemma \ref{lamma:project_to_surface}, we have
\begin{equation}
    \phi_1(x)=\underset{z\in C_1}{\inf}\Vert z-x\Vert_2\geq\underset{z\in C_2}{\inf}\Vert z-x\Vert_2=\phi_2(x). 
\end{equation}
If $x\in C_2$ but $x\notin C_1$, then $\phi_1(x)\leq 0\leq\phi_2(x)$. If $x\in C_1$, by Lemma \ref{lemma:deviate_from_surface}, then we have
\begin{equation}
    \overline{B(x,|\phi_1(x)|)}\subseteq C_1\subseteq C_2\Rightarrow \phi_1(x)\geq \phi_2(x).
\end{equation}
Therefore, we can conclude that $\phi_1(x)\geq\phi_2(x)$ in $\R^d$.

Conversely, if $\phi_1(x)\geq\phi_2(x)$ in $\Omega$, for any $x\in C_1$, $\phi_2(x)\leq\phi_1(x)\leq 0$ which implies that $C_1\subseteq C_2$. 
\end{proof}

\section{Two variational models with convexity constraint} \label{sec:model}
In this section, we will present the models for two applications involved 
convexity prior in details using the proposed convexity representation method. 
% In order to avoid the theories on the existence and uniqueness of the solutions in 
% continuous form, we only consider discrete models here. 
\subsection{Image segmentation with convexity prior}
Given a digital image $u(x)\in \{0,1,\cdots,255\}$ defined on a discrete image domain $\hat{\Omega}$, the goal of image segmentation is to partition it into $n$ disjoint sub-regions based on some features. To achieve this goal, many algorithms have been developed in the literature. Here we consider a 2-phase Potts model \cite{potts1952some} for segmentation:
\begin{equation}
    \underset{\{\hat{\Omega}_j\}_{j=0}^1}{\arg\min}  \sum_{j=0}^1 \sum_{x_i\in{\hat{\Omega}_j}}f_j(x_i)+\sum_{j=0}^1 G(\partial\Omega_i),
\end{equation}
where $f_j$ are the corresponding region force of each class and $G(\partial\hat{\Omega}_j)$ is the regularization term of the class boundaries. In the 2-phase cases, $\hat{\Omega}_1$ usually is the object of interest and $\hat{\Omega}_0$ is the background.

In this paper, we want to consider the segmentation problem with the convexity prior, i.e., the object is known to be convex.
Using the level-set representation and the 
results in last section, we can write the concerned segmentation problem as
\begin{align}
    \underset{\phi}{\arg\min} &\sum_{x_i\in\hat{\Omega}} (f_1(x_i)-f_0(x_i))h(\phi(x_i))+g(x_i)|\nabla h(\phi(x_i))|, \label{eq:seg_model0}\\
    \text{s.t. } & |\nabla\phi(x_i)|=1 \text{ in } \hat{\Omega},\\
    &H(\phi(x_i))\geq 0 \text{ in } L^-_\epsilon(|\phi|),
\end{align}
where $h(\cdot)$ denotes the Heaviside function and $g(x)=\frac{\alpha}{1+\beta|\nabla K*u(x)|}$ is and edge detector. $K$ denotes a smooth Gaussian kernel here. The first unitary gradient constraint can maintain $\phi$ to be a SDF and the second constraint is the convexity constraint. Since we only care about the convexity of the zero level-set, we can only impose this constraint in a neighbourhood around the zero level-set to save the computational effort.

There are many ways to define the data terms $f_1$ and $f_0$. In this work, we choose to use a probability-based region force term as in \cite{luo2018convex}, where the data terms are computed from some prior information. Suppose we know the labels of some pixels as prior knowledge, then we denote $I_1$ as the set of labeled points belonging to phase 1 and $I_0$ as the set of labeled points belonging to phase 0. Define the similarity between any two points in $\Omega$ as 
\begin{equation}
    S(x,y)=\exp\{-a\Vert x-y\Vert^2_2-b\Vert u(x)-u(y)\Vert_2^2\},
\end{equation}
where we will set $a=1$ and $b=10$ in this work. Then, the probability of a pixel $x$ belonging to phase 1 can be computed as
\begin{equation}
    p_1(x)=\frac{\sum_{y_i\in I_1}S(x,y_i)}{\sum_{y_i\in I_1\cup I_2}S(x,y_i)},
\end{equation}
 and the probability of belonging to phase 0 is $p_0(x)=1-p_1(x)$. The region force term is then defined as $f_i(x)=-\log(p_i(x)), i=1,2$. One can refer to \cite{luo2018convex} for more details about the segmentation model.
% In addition to the convexity prior, one may have some prelabeled information on the object and background. 
% Let $I_1$ (resp. $I_0$) be the set of labeled points belonging to phase 1 (resp. $I_0$), that is 
% the solution $\phi$ should satisfy $\phi(x)\leq0$ (resp. $\phi\geq0$) for $x\in I_0$ (resp. $x\in I_1$).  
Therefore, the model (\ref{eq:seg_model0}) can be rewritten as 
\begin{align}
    \underset{\phi}{\arg\min} &\sum_{x_i\in\hat{\Omega}} \left [-\log(p_1(x_i))+\log(p_0(x_i))\right ]h(\phi(x_i))+\mu g(x_i)|\nabla h(\phi(x_i))| \label{eq:seg_model}\\
    \text{s.t. } & |\nabla\phi(x_i)|=1 \text{ in } \hat{\Omega},\\
    & H(\phi(x_i))\geq 0 \text{ in } L^-_\epsilon(|\phi|),\\
    & \phi(x_i)\leq 0 \text{ for any } x_i\in I_1,\quad \phi(x_i)\geq 0 \text{ for any } x_i\in I_0.
\end{align}

%For some applications, we can incorporate convexity prior to improve the segmentation accuracy. Combining with the proposed convexity constraint, we can write the convex segmentation model as
% \begin{align}
%     \underset{\phi}{\arg\min} &\sum_{x_i\in\hat{\Omega}} \left [-\log(p_1(x_i))+\log(p_0(x_i))\right ]h(\phi(x_i))+\mu g(x_i)|\nabla h(\phi(x_i))| \label{eq:seg_model}\\
%     \text{s.t. } & |\nabla\phi(x_i)|=1 \text{ in } \hat{\Omega},\\
%     & H(\phi(x_i))\geq 0 \text{ in } L^-_\epsilon(|\phi|),\\
%     & \phi(x_i)\leq 0 \text{ for any } x_i\in I_1,\quad \phi(x_i)\geq 0 \text{ for any } x_i\in I_0.
% \end{align}

The numerical descritization scheme for the differential operators $\nabla(\cdot)$ and $H(\cdot)$ will be introduced later in Section \ref{sec:algo}. Similar to \cite{li2019convex}, we assume that the input image $u$ is periodic in $\R^d$ which implies that $\phi$ satisfies the periodic boundary condition on $\partial\Omega$. Using the fact that $|\nabla h(\phi)|=\delta(\phi)|\nabla\phi|=\delta(\phi)$, the last term in the objective functional can be written as $g(x)\delta(\phi(x))$ where $\delta(\cdot)$ is the Dirac delta function. 
In the numerical computation, we will replace $h(\cdot)$ and $\delta(\cdot)$ by their smooth approximations: 
\begin{align}
    &h(y)\approx h_\alpha(y)=\frac{1}{2}+\frac{1}{\pi}\arctan(\frac{y}{\alpha}),\\
    &\delta(y)\approx h'_\alpha(y)=\delta_\alpha(y)=\frac{\alpha}{\pi(y^2+\alpha^2)},
\end{align}
where $\alpha$ is a small positive number. The segmentation model (\ref{eq:seg_model}) can also be directly generalized to high dimensional cases for object segmentation.

\subsection{Convex hull model}
Suppose we are given a hyper-rectangular domain $\Omega\subset\R^d$ and we want to find the convex hull of a subset $X\subset \Omega$. From the definition, we know that the convex hull is the smallest convex set containing $X$. If we denote the set of all SDFs of convex subset in $\Omega$ as $\mathbb{K}$, the SDF corresponding to the convex hull minimizes the zero sub-level set area (or volume) among $\mathbb{K}$: 
\begin{equation}\min_{\phi\in\mathbb{K}}\int_{\Omega}(1-h(\phi(x)))dx.\end{equation}
By Lemma \ref{lemma:smallest_SDF}, we can obtain an equivalent and simpler form:
\begin{equation}\min_{\phi\in\mathbb{K}}\int_{\Omega}-\phi(x)dx,\end{equation}
which leads to the following discrete problem:
\begin{align}
    \underset{\phi}{\min}\quad &\sum_{x_i\in\hat{\Omega}}-\phi(x_i) \label{eq:exact_convexhull}\\
    \text{s.t.}\quad & |\nabla\phi(x_i)|=1 \text{ in } \hat{\Omega},\\
    & H(\phi(x_i))\geq 0 \text{ in } L^-_{\epsilon}(|\phi|), \label{eq:constraint_convexity}\\
    & \phi(x_i)\leq 0 \text{ for any } x_i\in X.
 \end{align}
This model is a simplified version of the convex hull models in \cite{li2019convex} and can be applied to any dimensions. Here the first two constraints are the same with the segmentation model (\ref{eq:seg_model}). The last constraint requires the zero sublevel-set of $\phi$ contain $X$. Again, we assume $\phi$ satisfies the periodic boundary condition. 

As we mentioned before, when the input data contains outliers, it is not appropriate to require $L_0^-(\phi)$ to enclose all the given data in $X$. Instead, we can use a penalty function $R(x)$ to penalize large $\phi(x)$ for all $x\in X$. By selecting appropriate parameters, we can find an approximated convex hull of the original set with high accuracy. The approximation model is given as:
\begin{align}
    \underset{\phi}{\min}\quad &\sum_{x_i\in\hat{\Omega}}-\phi(x_i)+\lambda m(x_i)R(\phi(x_i)) \label{eq:approximate_convexhull}\\
    \text{s.t.}\quad & |\nabla\phi(x_i)|=1 \text{ in } \hat{\Omega},\\
    & H(\phi(x_i))\geq 0 \text{ in } L^-_\epsilon(|\phi|).
 \end{align}
where $m(x)$ equals to 1 in $X$ and 0 elsewhere. Here we can choose $R$ to be the positive part function 
\begin{equation}
    R(s)=\begin{cases}s & s>0 \\ 0 & s\leq 0,
    \end{cases}
\end{equation}
or its smooth approximation $R(s)=\frac{1}{t}\log(1+e^{ts})$ for some $t>0$.

\section{Numerical algorithm} \label{sec:algo}
Here we propose an efficient algorithm based on the alternating direction method of multipliers (ADMM) to solve the segmentation model (\ref{eq:seg_model}) and the convex hull models (\ref{eq:exact_convexhull}) (\ref{eq:approximate_convexhull}). We will also introduce the descritization scheme for the differential operators.

\subsection{ADMM algorithm for the segmentation model and convex hull models}
First, we write the three models (\ref{eq:seg_model}), (\ref{eq:exact_convexhull}) and (\ref{eq:approximate_convexhull}) in a unified framework:
\begin{align}
     \min_{\phi}\  &F(\phi), \label{eq:unified_model}\\
     \text{s.t. }&|\phi(x_i)|=1 \text{ in } \hat{\Omega},\\
     & H(\phi(x_i))\geq 0 \text{ in } L^-_{\epsilon}(|\phi|),\\
     & \phi\in\mathbb{S},
\end{align}
where $F(\phi)$ could be the objective functional in (\ref{eq:seg_model}), (\ref{eq:exact_convexhull}) or (\ref{eq:approximate_convexhull}), and  $\mathbb{S}$ is defined as
\begin{equation}
    \{\psi:\hat{\Omega}\rightarrow\mathbb{R}|\psi(x_i)\leq 0\text{ in } I_1, \psi(x_i)\geq 0 \in I_0\}.
\end{equation} 
For the segmentation model (\ref{eq:seg_model}), $I_0$ and $I_1$ are those labelled points. For the exact convex hull model (\ref{eq:exact_convexhull}), $I_0$ is the set $X$ and $I_1$ is empty. For the approximate convex hull model (\ref{eq:approximate_convexhull}), both $I_0$ and $I_1$ could be empty. If we have any prior labels for the approximate convex hull model, we can also incorporate them into the models to make $I_0$ and $I_1$ non-empty. 

To solve model (\ref{eq:unified_model}), we introduce three auxiliary variables $p=\nabla\phi$, $Q=H(\phi)$, and $z=\phi$. Then we can write the augmented Lagrangian functional as
\begin{align}
    &L(\phi,p,Q,z,\gamma_1,\gamma_2,\gamma_3)=\sum_{x_i\in\hat{\Omega}}\{-\phi(x_i)+\frac{\rho_1}{2}|\nabla\phi(x_i)-p(x_i)|^2+\frac{\rho_2}{2}\Vert H(\phi(x_i))-Q(x_i)\Vert_F^2\\
    &\quad +\frac{\rho_3}{2}|\phi(x_i)-z(x_i)|^2+\langle\gamma_1(x_i),\nabla\phi(x_i)-p(x_i)\rangle+\langle\gamma_2(x_i),H(\phi(x_i))-Q(x_i)\rangle_F\notag\\
    &\quad +\gamma_3(x_i)(\phi(x_i)-z(x_i))\}\notag
\end{align}
with the following constraints: $|p|=1$ in $\hat{\Omega}$, $Q\geq 0$ in $L_\epsilon^-(|\phi|)$ and $z\in\mathbb{S}$, where $\mathbb{S}$ 
is specified by the problems on hand. 
% are defined as
% \begin{equation}
%     \{\psi:\hat{\Omega}\rightarrow\mathbb{R}|\psi(x)\leq 0\text{ in }X\}.
% \end{equation} 
Applying the ADMM algorithm, we can split the original problem into several sub-problems.
\begin{enumerate}
    \item $\phi$ sub-problem:
    \begin{align}
        \phi^{t+1}=&\underset{\phi}{\arg\min}\  L(\phi,p^t,Q^t,z^t,\gamma_1^t,\gamma_2^t,\gamma_3^t)\notag\\
        =&\underset{\phi}{\arg\min}\   \sum_{x_i\in\hat{\Omega}}-\phi+\frac{\rho_1}{2}|\nabla\phi-p^t|^2+\frac{\rho_2}{2}\Vert H(\phi)-Q^t\Vert_F^2+\frac{\rho_3}{2}|\phi-z^t|^2\notag\\
        &+\langle\gamma_1^t,\nabla\phi-p^t\rangle+\langle\gamma_2^t,H(\phi)-Q^t\rangle_F+\gamma_3^t(\phi-z^t).\notag
    \end{align}
    Then the optimal $\phi$ must satisfy
    \begin{align}
        \nabla^*(\rho_1(\nabla\phi-p^t)+\gamma_1^t)+H^*(\rho_2(H(\phi)-Q^t)+\gamma_2^t)+(\rho_3(\phi-z^t)+\gamma_3^t)=1,\notag
    \end{align}
    where $\nabla^*(\cdot)$ is the adjoint operator of $\nabla(\cdot)$ and $H^*(\cdot)$ is the adjoint operator of $H(\cdot)$. It is equivalent to 
    \begin{align}
        -\rho_1\Delta\phi+\rho_2H^*(H(\phi))+\rho_3\phi=rhs^t_1,
    \end{align}
    where $rhs^t_1$ is a constant vector here. Notice that $H^*(H(\phi))=\Delta^2\phi$, and then the fourth-order PDE can be rewritten as
    \begin{align}
        \rho_2\Delta^2\phi-\rho_1\Delta\phi+\rho_3\phi=rhs^t_1, \label{eq:phi_pde_1}
    \end{align}
    where $rhs_t=1-\nabla^*(-\rho_1p^t+\gamma_1^t)-H^*(-\rho_2Q^t+\gamma_2^t)-(-\rho_3z^t+\gamma_3^t)$ is a constant vector here. Since $\phi$ satisfies the periodic boundary condition, we can apply FFT as \cite{li2019convex} to solve it.
    
    \item $p$ sub-problem:
    \begin{align}
        p^{t+1}=&\underset{|p|=1}{\arg\min}\  L(\phi^{t+1},p,Q^t,z^t,\gamma_1^t,\gamma_2^t,\gamma_3^t)\notag\\
        =&\underset{|p|=1}{\arg\min}\  \sum_{x_i\in\hat{\omega}}\frac{\rho_1}{2}|\nabla\phi^{t+1}-p|^2+\langle\gamma_1^t,\nabla\phi^{t+1}-p\rangle\notag\\
        =&\frac{\gamma_1^t/\rho_1+\nabla\phi^{t+1}}{|\gamma_1^t/\rho_1+\nabla\phi^{t+1}|}.\label{eq:p_update}
    \end{align}
    \item $Q$ sub-problem:
    \begin{align}
        Q^{t+1}=&\underset{Q\in L_\epsilon^-(|\phi|)}{\arg\min}\ L(\phi^{t+1},p^{t+1},Q,z^t,\gamma_1^t,\gamma_2^t,\gamma_3^t)\notag\\
        =&\underset{Q\in L_\epsilon^-(|\phi|)}{\arg\min}\ \sum_{x_i\in\hat{\Omega}}\frac{\rho_2}{2}\Vert H(\phi^{t+1})-Q\Vert_F^2+\langle \gamma_2^t,H(\phi^{t+1})-Q\rangle_F\notag\\
        =&\begin{cases}
        \gamma_2^t/\rho_2+H(\phi^{t+1}) & \text{ for } x\notin L_\epsilon^-(|\phi|)\\
        \Pi_{\text{PSD}}\{\gamma_2^t/\rho_2+H(\phi^{t+1})\} &\text{ for }x\in L_\epsilon^-(|\phi|)
        \end{cases}.\label{eq:Q_update}
    \end{align}
    The computation of this projection is very simple and can be done by the same way with \cite{rosman2014fast}. Given a real symmetric matrix $A\in\mathbb{R}^{n\times n}$, suppose it's singular value decomposition is $A=Q\Lambda Q^T$, where $Q$ is orthonormal and $\Lambda=\text{diag}(\lambda_1,\dots,\lambda_n)$ is a diagonal matrix with all the singular values of $A$ on its diagonal entries. We further assume that $\lambda_1,\dots,\lambda_m\geq 0$ and $\lambda_{m+1},\dots,\lambda_n<0$. If we define 
    \begin{equation}\Lambda^+=\text{diag}(\max \{ \lambda_1,0\},\dots,\max\{\lambda_n,0\}),\end{equation} 
    then the projection of $A$ onto the set of all symmetric positive semi-definite (SPSD) matrices under the Frobenius norm is $A^*=Q\Lambda^+Q^T$. The proof can be found in \cite{higham1988matrix}.

    % \begin{proof}
    % From the projection theorem, since the set of SPSD matrices is convex, we only need to show that for any SPSD matrix $X$, the following inequality holds:
    % \begin{equation}
    %     \langle A-A^*,X-A^*\rangle_F\leq 0.
    % \end{equation}
    % Since $\langle X,Y\rangle_F=\text{trace}(X^TY)$, we can derive that
    % \begin{align}
    %     &\langle Q^TXQ,Q^TYQ\rangle_F\\
    %     =&\text{trace}(Q^TX^TQQ^TYQ)\\
    %     =&\text{trace}(X^TY)\\
    %     =&\langle X,Y\rangle_F.
    % \end{align}
    % Therefore,
    % \begin{align}
    %     &\langle A-A^*,X-A^*\rangle_F\\
    %     =&\langle \Lambda-\Lambda^+,Q^TXQ-\Lambda^+\rangle_F\\
    %     =&\langle \Lambda-\Lambda^+,Y-\Lambda^+\rangle_F\\
    %     =&\sum_{i=m+1}^n\lambda_iY_{ii},
    % \end{align}
    % where $Y=Q^TXQ$. Since $Y$ is also SPSD, its diagonal elements are all non-negative. Thus, we have
    % \begin{equation}
    %     \langle A-A^*,X-A^*\rangle_F=\sum_{i=m+1}^n\lambda_iY_{ii}\leq 0.
    % \end{equation}
    % \end{proof}
    
    \item $z$ sub-problem:
    \begin{align}
        z^{t+1}=&\underset{z\in\mathbb{S}}{\arg\min}\ L(\phi^{t+1},p^{t+1},Q^{t+1},z,\gamma_1^t,\gamma_2^t,\gamma_3^t)\\
        =&\underset{z\in\mathbb{S}}{\arg\min} \sum_{x_i\in\hat{\Omega}}\frac{\rho_3}{2}|\phi^{t+1}-z|^2+\gamma_3^t(\phi^t-z)\\
        =&\begin{cases}
        \min\{0,\gamma_3^t/\rho_3+\phi^{t+1}\} &x\in I_1\\
        \max\{0,\gamma_3^t/\rho_3+\phi^{t+1}\} &x\in I_0\\
        \gamma_3^t/\rho_3+\phi^{t+1} &x\notin I_1\cup I_0
        \end{cases}.\label{eq:z_update}
    \end{align}
\end{enumerate}
We summarize this procedure in Algorithm \ref{algo:ADMM}.
\begin{algorithm}[ht]
\caption{ADMM algorithm for variational models with convexity prior}
\begin{algorithmic}[1]
    \State Input the parameters $\rho_1$, $\rho_2$ and $\rho_3$.
    \State Initialize $\phi^0$ to be the SDF of an initial shape. Initialize $p^0$, $Q^0$, $z^0$, $\gamma_1^0$, $\gamma_2^0$ and $\gamma_3^0$ to be all zeros.
    \While{stopping criterion is not satisfied}
        \State compute $\phi^{t+1}$ by solving the PDE (\ref{eq:phi_pde_1}) which can be done by applying FFT on both side of the equation as \cite{li2019convex}.
        \State compute $p^{t+1}$ by (\ref{eq:p_update}).
        \State compute $Q^{t+1}$ by (\ref{eq:Q_update}).
        \State compute $z^{t+1}$ by (\ref{eq:z_update}).
    \EndWhile
\end{algorithmic}\label{algo:ADMM}
\end{algorithm}

\subsection{Numerical descritization scheme}
Suppose we are given a volumetric data $I\in\R^{N_1\times\dots\times N_d}$ which is a binary function defined on the discrete domain $\Omega_h$. $I$ can also be viewed as a characteristic function of a subset $X_h\subseteq\Omega_h$. Each mesh point in $\Omega_h$ can be represented by a d-tuple $\mathbf{x}=(x_1,\dots,x_d)$ where $x_i$ is an integer in $[1,N_i]$ for $i=1,\dots,d$. To compute the convex hull of $X_h$ using the algorithm described before, we need first to approximate the differential operators numerically. For any function $\phi:\Omega_h\rightarrow\R^{N_1\times\dots\times N_d}$, we denote $\partial_i^+(\mathbf{x})$ and $\partial_i^-(\mathbf{x})$ as
\begin{align}
    &\partial_i^+\phi(\mathbf{x})=\begin{cases}\phi(x_1,\dots,x_i+1,\dots,x_d)-\phi(x_1,\dots,x_i,\dots,x_d) & x_i< N_i\\
    \phi(x_1,\dots,x_1,\dots,x_d)-\phi(x_1,\dots,x_{N_i},\dots,x_d) & x_i=N_i
    \end{cases},\\
    &\partial_i^-\phi(\mathbf{x})=\begin{cases}\phi(x_1,\dots,x_i,\dots,x_d)-\phi(x_1,\dots,x_i-1,\dots,x_d) &x_i>1\\
    \phi(x_1,\dots,x_1,\dots,x_d)-\phi(x_1,\dots,x_{N_i},\dots,x_d) &x_i=1
    \end{cases}.
\end{align}
These are exactly the standard forward difference and backward difference. Then we approximate the first and second order operators as follow:
\begin{align}
    & \nabla\phi\approx\tilde{\nabla}\phi=(\partial^+_1\phi,\dots,\partial^+_d\phi)^T,\\
    & \nabla^*p\approx\tilde{\nabla}^*\phi=-\sum_{i=1}^d \partial^-_i p_i,
\end{align}
and
\begin{align}
    &\Delta\phi\approx\tilde{\Delta}\phi=-\tilde\nabla^*(\tilde{\nabla}\phi), \\
    & H(\phi)_{ij}\approx\tilde{H}(\phi)_{ij}=\begin{cases}
    \partial^-_i(\partial^+_j\phi) & \text{ if } i=j \\
    \partial^+_i(\partial^+_j\phi)& \text{ if } i\neq j
    \end{cases},\\ 
    & H^*(Q)\approx\tilde{H}^*(Q)=\sum_{i=1}^n \partial^-_i\partial^+_i Q_{ii}+\sum_{i\neq j}\partial^-_i\partial^-_j Q_{ij}. 
\end{align}
One can verify that the relation $\tilde{H}^*(\tilde{H}(\phi))=\tilde{\Delta}^2(\phi)$ is also preserved by this discretization scheme. 

\section{Numerical experiment} \label{sec:exp}
In this section, we are going to demonstrate many numerical experiments of the segmentation models and convex hull models.

\subsection{Image segmentation and object segmentation}
We first test our segmentation model for some images from \cite{luo2018convex}. In Figure \ref{fig:cell}, we test the segmentation model (\ref{eq:seg_model}) on a tumor image (a). Since the model needs some prior information, we give the prior labels in (b) and the initial curve in (c). One can observe that there exists a dark region in this tumor. If we do not impose the convexity constraint, the result we get is shown in (d) where the dark region is missing. After imposing the convexity constraint, we can recover the whole tumor region. The result of our algorithm is shown in (f) which is very similar to the result in \cite{luo2018convex} (e).

\begin{figure}[ht]
    \centering
    \subfloat[original image]{\includegraphics[width=2.5cm]{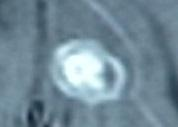}}\quad
    \subfloat[prior labels]{\includegraphics[width=2.5cm]{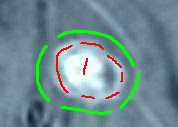}}\quad
    \subfloat[initial curve]{\includegraphics[width=2.5cm]{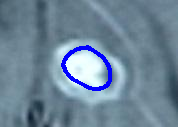}}\\
    \subfloat[without convexity constraint]{\includegraphics[width=2.5cm]{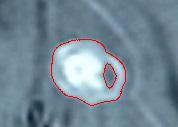}}\quad
    \subfloat[with convexity constraint by \cite{luo2018convex}]{\includegraphics[width=2.5cm]{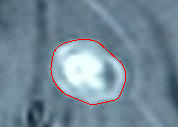}}\quad
    \subfloat[with our convexity constraint]{\includegraphics[width=2.5cm]{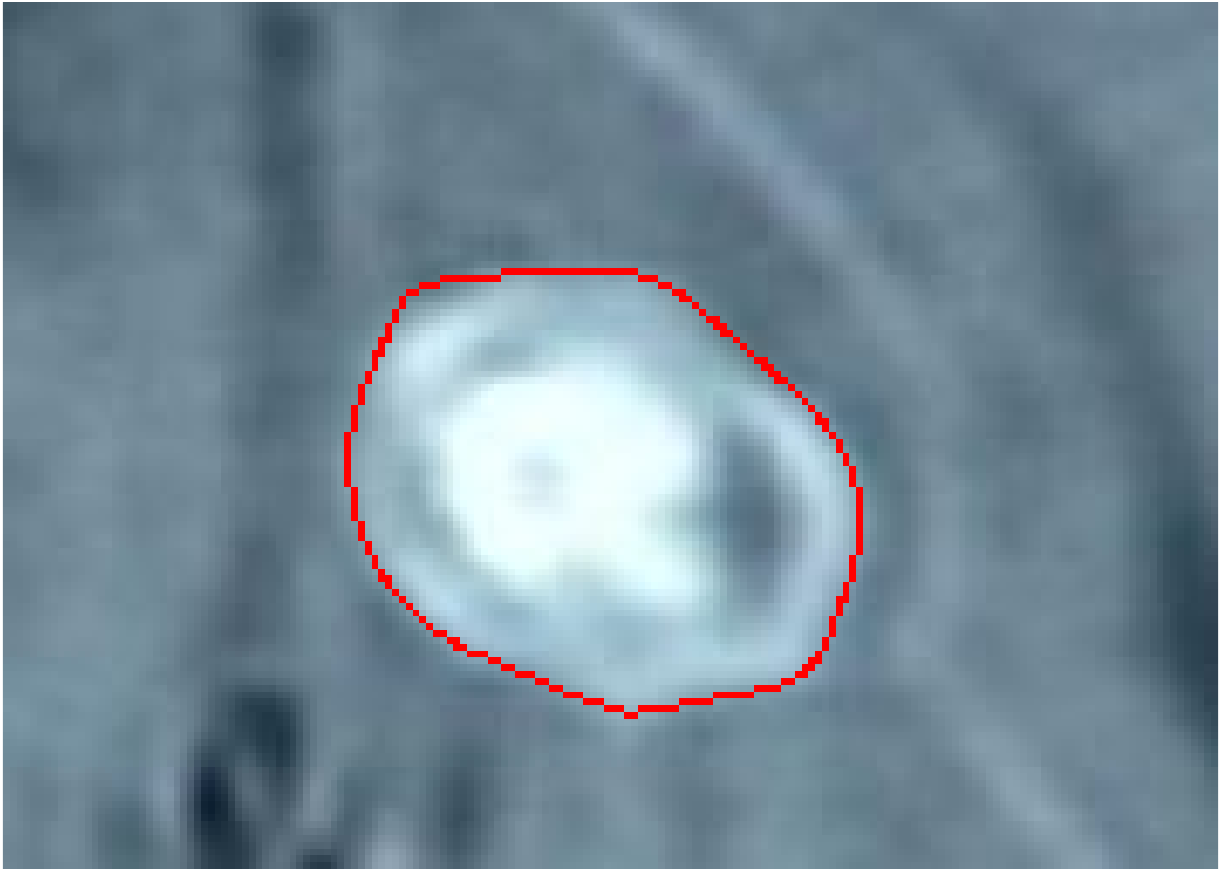}}
    \caption{Segmentation results of a tumor with and without convexity constraint.}
    \label{fig:cell}
\end{figure}

We also test the segmentation model on a 3-D data of brain tumor from \cite{isensee2017brain}. The size of this volume is $240\times 240\times 152$. In Figure \ref{fig:tumor3d_original}, we show some cross sections of the original volume. Then we give some prior labels at these 9 slices as in Figure \ref{fig:tumor3d_prior} to compute the region force term. The initial region is taken as the set where the region force is positive. One can observe that the initial region shows some concavity in several slices, e.g, $z=90$. We then compute the segmentation result using our proposed method. The 3-d visualization of the segmentation result is shown in Figure \ref{fig:tumor3d_vis} and some cross sections are presented in Figure \ref{fig:tumor3d_results}. Though the initial shape is not convex, we can see that the convexity method can obtain a convex shape. What's more, the tumor region is also identified accurately.

\begin{figure}[ht]
    \centering
    \subfloat[$z=30$]{\includegraphics[width=2cm]{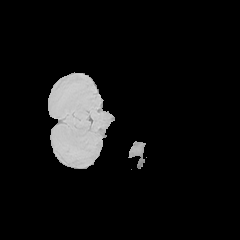}}\quad
    \subfloat[$z=40$]{\includegraphics[width=2cm]{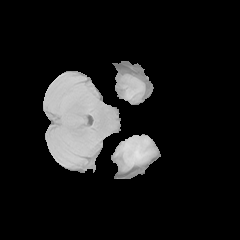}}\quad
    \subfloat[$z=50$]{\includegraphics[width=2cm]{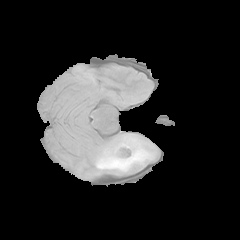}}\\
    \subfloat[$z=60$]{\includegraphics[width=2cm]{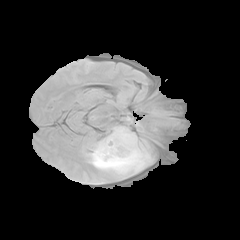}}\quad
    \subfloat[$z=70$]{\includegraphics[width=2cm]{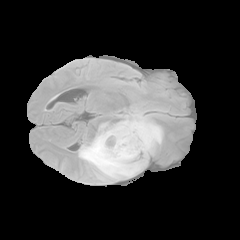}}\quad
    \subfloat[$z=80$]{\includegraphics[width=2cm]{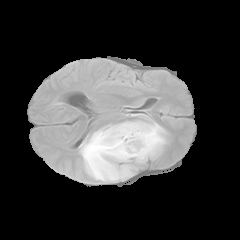}}\\
    \subfloat[$z=90$]{\includegraphics[width=2cm]{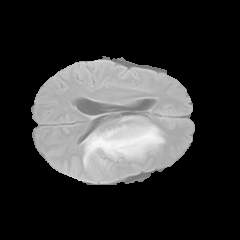}}\quad
    \subfloat[$z=100$]{\includegraphics[width=2cm]{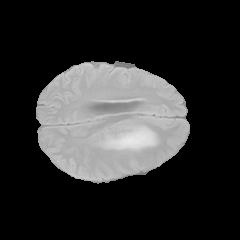}}\quad
    \subfloat[$z=110$]{\includegraphics[width=2cm]{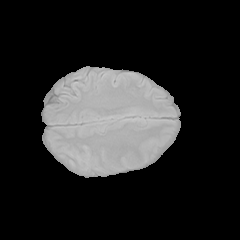}}
    \caption{Original 3-D volume of a brain tumor}
    \label{fig:tumor3d_original}
\end{figure}

\begin{figure}[p]
    \centering
    \subfloat[$z=30$]{\includegraphics[width=2cm]{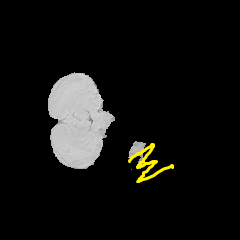}}\quad
    \subfloat[$z=40$]{\includegraphics[width=2cm]{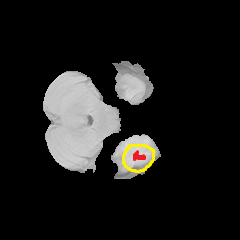}}\quad
    \subfloat[$z=50$]{\includegraphics[width=2cm]{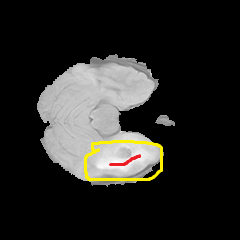}}\\
    \subfloat[$z=60$]{\includegraphics[width=2cm]{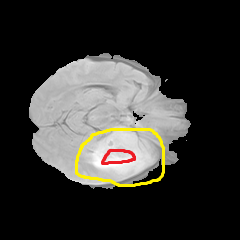}}\quad
    \subfloat[$z=70$]{\includegraphics[width=2cm]{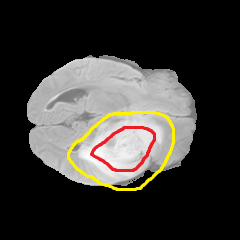}}\quad
    \subfloat[$z=80$]{\includegraphics[width=2cm]{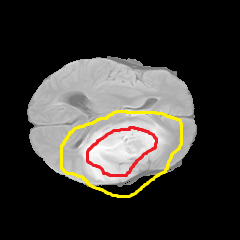}}\\
    \subfloat[$z=90$]{\includegraphics[width=2cm]{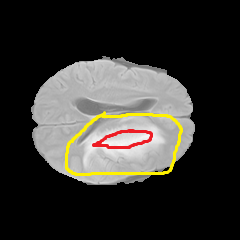}}\quad
    \subfloat[$z=100$]{\includegraphics[width=2cm]{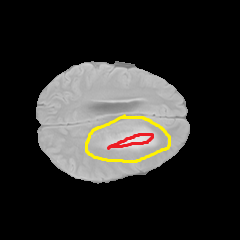}}\quad
    \subfloat[$z=110$]{\includegraphics[width=2cm]{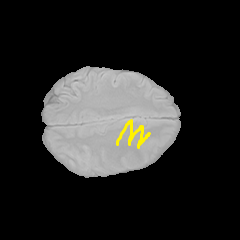}}
    \caption{Prior labels assigned for the tumor. Red points represent foreground labels and yellow points represent background labels. }
    \label{fig:tumor3d_prior}
\end{figure}

\begin{figure}[p]
    \centering
    \subfloat[$z=30$]{\includegraphics[width=2cm]{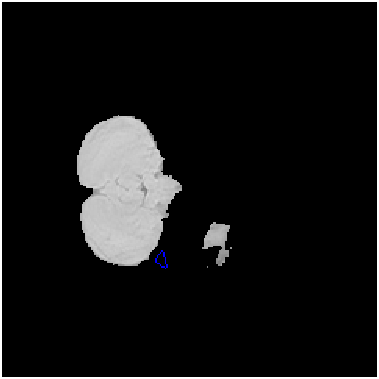}}\quad
    \subfloat[$z=40$]{\includegraphics[width=2cm]{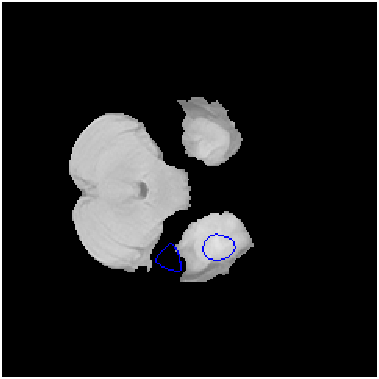}}\quad
    \subfloat[$z=50$]{\includegraphics[width=2cm]{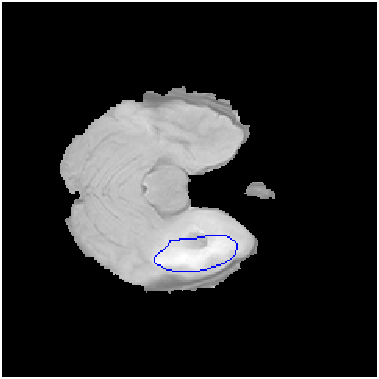}}\\
    \subfloat[$z=60$]{\includegraphics[width=2cm]{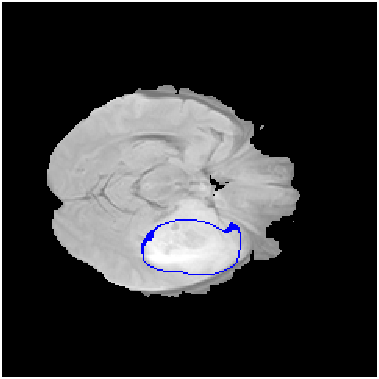}}\quad
    \subfloat[$z=70$]{\includegraphics[width=2cm]{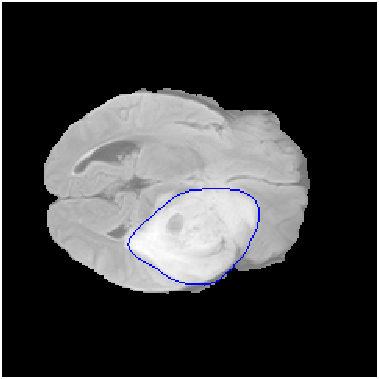}}\quad
    \subfloat[$z=80$]{\includegraphics[width=2cm]{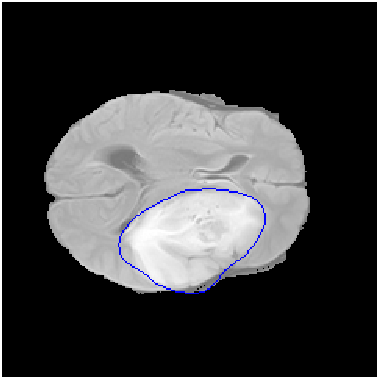}}\\
    \subfloat[$z=90$]{\includegraphics[width=2cm]{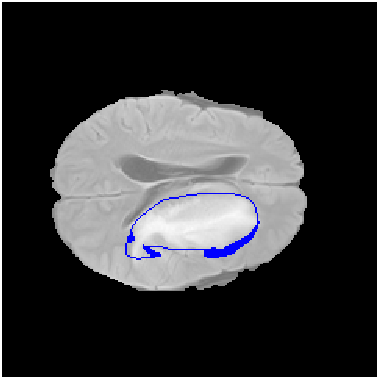}}\quad
    \subfloat[$z=100$]{\includegraphics[width=2cm]{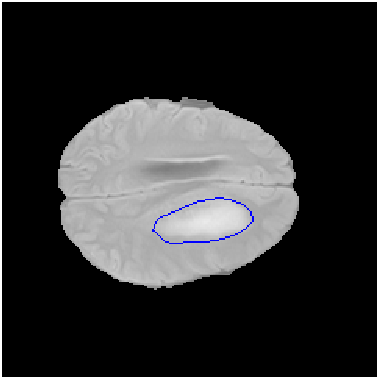}}\quad
    \subfloat[$z=110$]{\includegraphics[width=2cm]{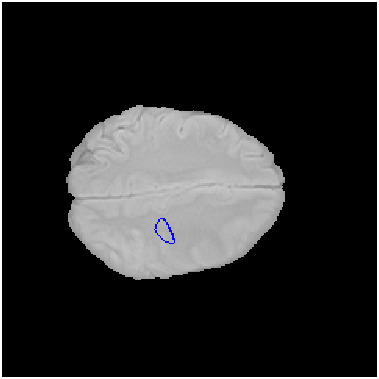}}
    \caption{Some cross sections of the initial shape for the model. The initial shape is the set of points at which the region force term is positive.}
    \label{fig:initial_curve}
\end{figure}

\begin{figure}[p]
    \centering
    \subfloat[]{\includegraphics[width=3cm]{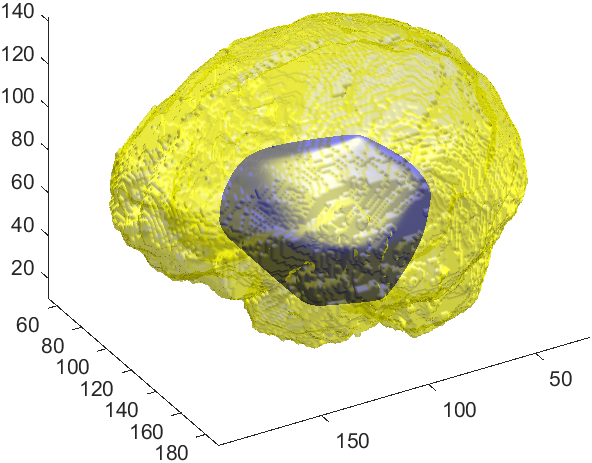}}\quad
    \subfloat[]{\includegraphics[width=3cm]{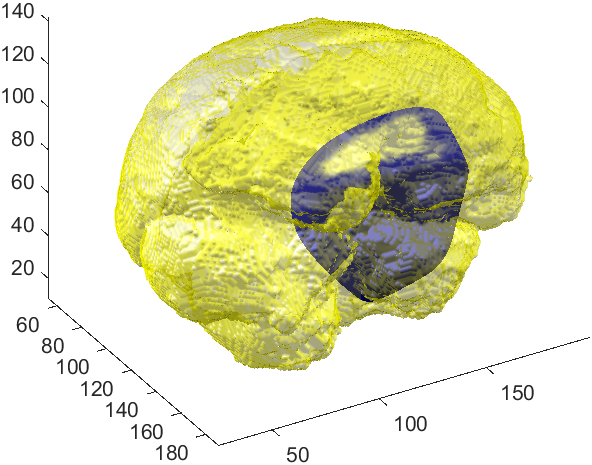}}\quad
    \subfloat[]{\includegraphics[width=3cm]{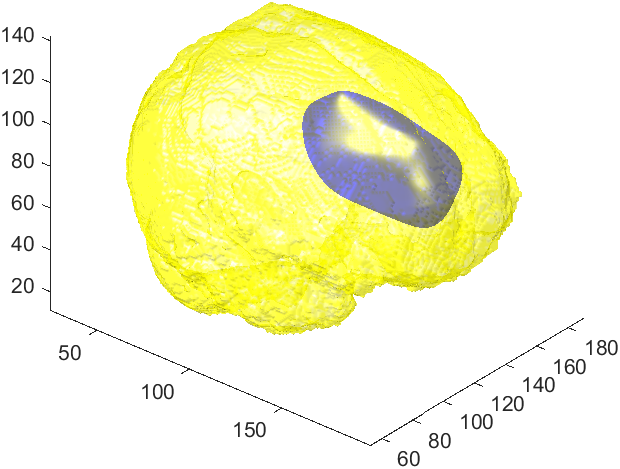}}
    \caption{Different views of the 3D visualization of the tumor.}
    \label{fig:tumor3d_vis}
\end{figure}

\begin{figure}[p]
    \centering
    \subfloat[$z=40$]{\includegraphics[width=2cm]{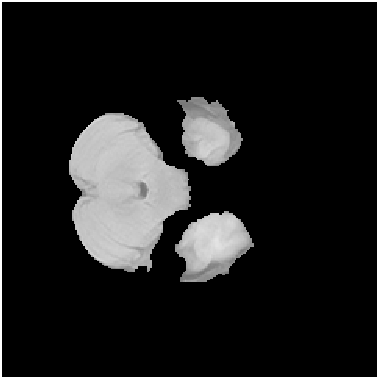}}\quad
    \subfloat[$z=50$]{\includegraphics[width=2cm]{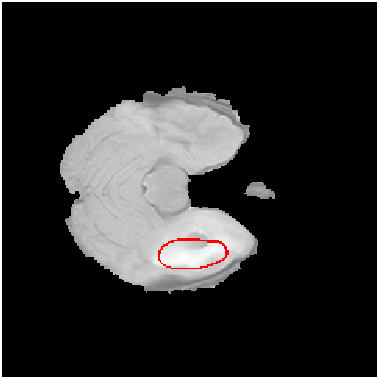}}\quad
    \subfloat[$z=60$]{\includegraphics[width=2cm]{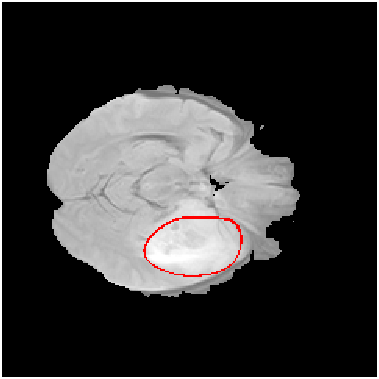}}\quad
    \subfloat[$z=70$]{\includegraphics[width=2cm]{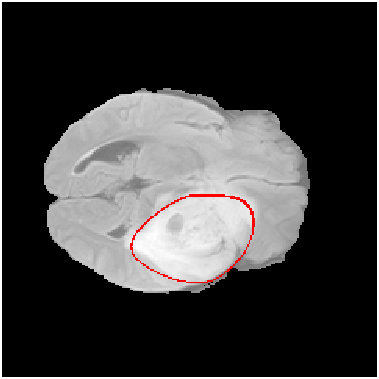}}\\
    \subfloat[$z=80$]{\includegraphics[width=2cm]{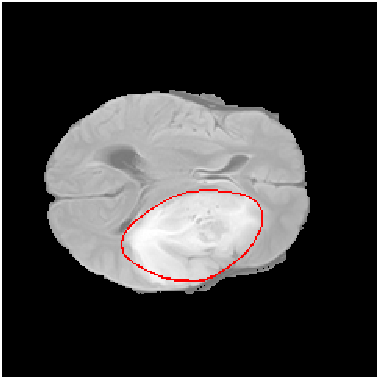}}\quad
    \subfloat[$z=90$]{\includegraphics[width=2cm]{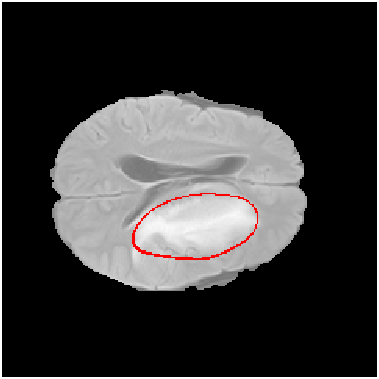}}\quad
    \subfloat[$z=100$]{\includegraphics[width=2cm]{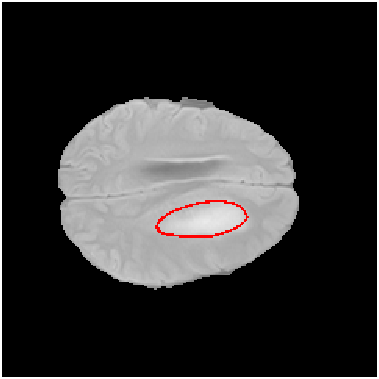}}\quad
    \subfloat[$z=110$]{\includegraphics[width=2cm]{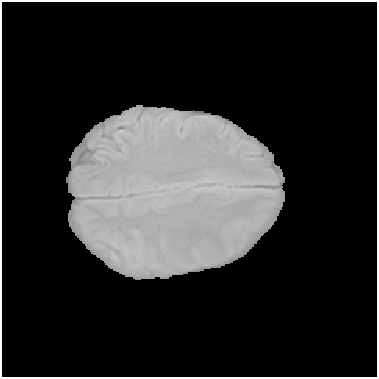}}\\
    \subfloat[$y=90$]{\includegraphics[width=2cm]{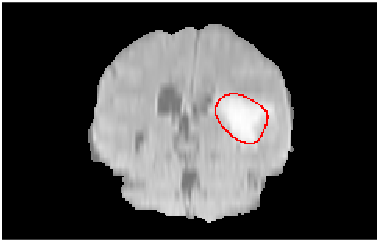}}\quad
    \subfloat[$y=100$]{\includegraphics[width=2cm]{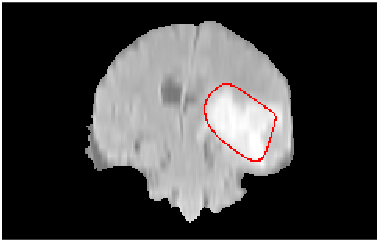}}\quad
    \subfloat[$y=110$]{\includegraphics[width=2cm]{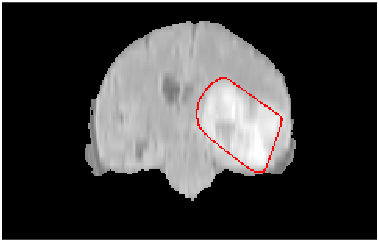}}\quad
    \subfloat[$y=120$]{\includegraphics[width=2cm]{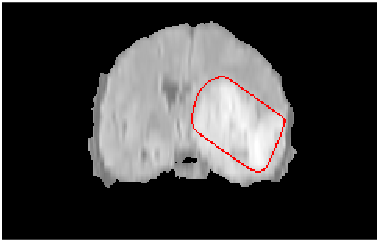}}\\
    \subfloat[$y=130$]{\includegraphics[width=2cm]{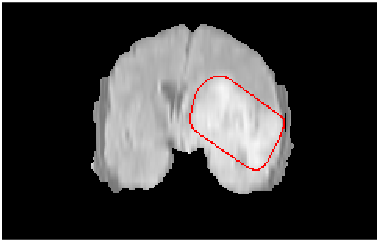}}\quad
    \subfloat[$y=140$]{\includegraphics[width=2cm]{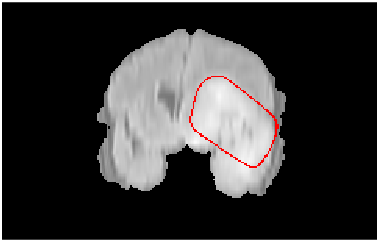}}\quad
    \subfloat[$y=150$]{\includegraphics[width=2cm]{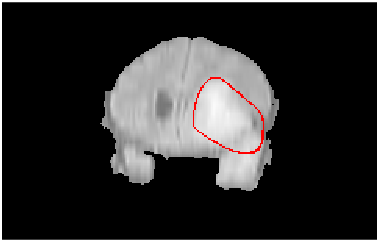}}\quad
    \subfloat[$y=160$]{\includegraphics[width=2cm]{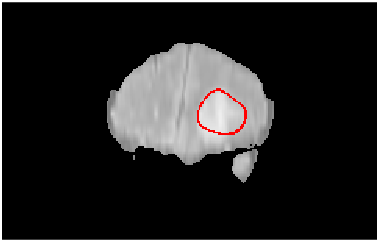}}\\
    \subfloat[$x=120$]{\includegraphics[width=2cm]{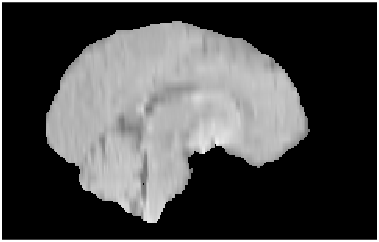}}\quad
    \subfloat[$x=130$]{\includegraphics[width=2cm]{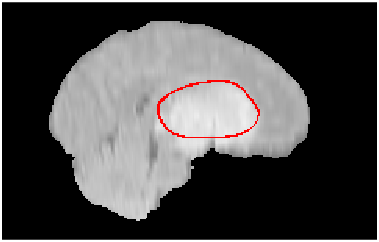}}\quad
    \subfloat[$x=140$]{\includegraphics[width=2cm]{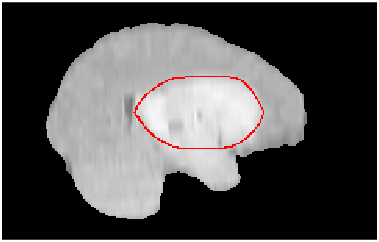}}\quad
    \subfloat[$x=150$]{\includegraphics[width=2cm]{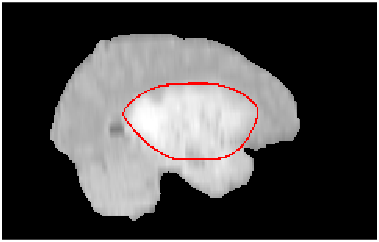}}\\
    \subfloat[$x=160$]{\includegraphics[width=2cm]{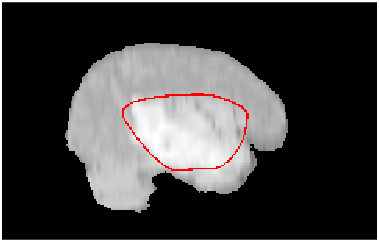}}\quad
    \subfloat[$x=170$]{\includegraphics[width=2cm]{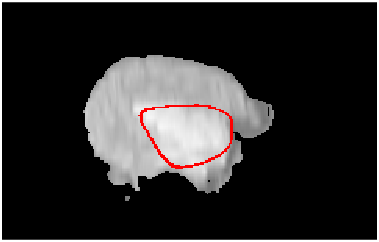}}\quad
    \subfloat[$x=180$]{\includegraphics[width=2cm]{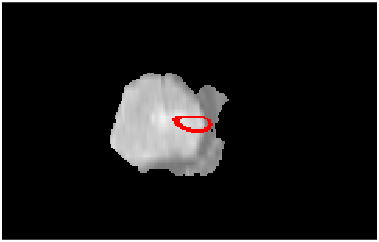}}\quad
    \subfloat[$x=185$]{\includegraphics[width=2cm]{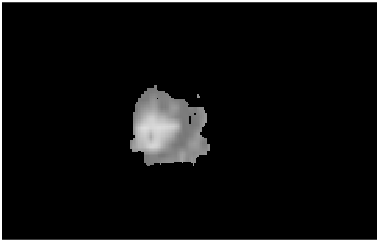}}
    \caption{Different cross sections of the segmentation result of the brain tumor.}
    \label{fig:tumor3d_results}
\end{figure}

\subsection{Convex hull model}
First we use three images from \cite{AlpertGBB07} to test our exact convex hull model for 2-d data. In \cite{li2019convex}, the authors have conducted many experiments on the same dataset. Here we demonstrate the results of 3 examples in Figure \ref{fig:2d_convhulls}. The numerical results are very closed to \cite{li2019convex} and many other traditional methods. We also compute the relative distance error of our method compared to the quickhull algorithm. The error measure is defined in the same way with \cite{li2019convex}:
\begin{align}
    &\text{err}(C_2)=\frac{\text{dist}_H(C_1,C_2)}{R},\label{eq:error}
\end{align}
where $R$ is the equivalent radius and is defined as the radius of the ball with same area with $C_1$. The error of our method and \cite{li2019convex} are listed in Table \ref{table:error} where you can see the results are very closed. 
\begin{figure}[ht]
    \centering
    \subfloat[boat]{\includegraphics[width=2.5cm]{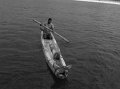}}\quad 
    \subfloat[helicopter]{\includegraphics[width=2.5cm]{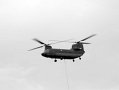}}\quad 
    \subfloat[Tendrils]{\includegraphics[width=2.5cm]{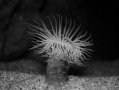}}\\
    \subfloat[boat]{\includegraphics[width=2.5cm]{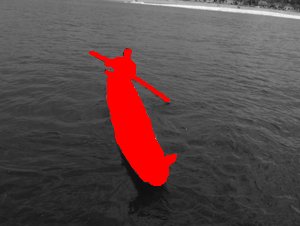}}\quad 
    \subfloat[helicopter]{\includegraphics[width=2.5cm]{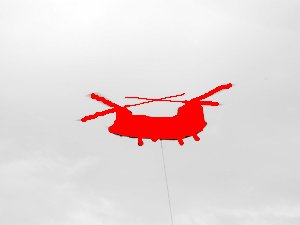}}\quad 
    \subfloat[Tendrils]{\includegraphics[width=2.5cm]{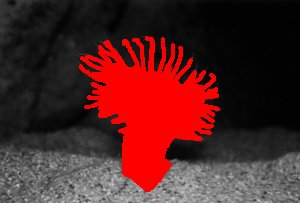}}\\
    \subfloat[boat]{\includegraphics[width=2.5cm]{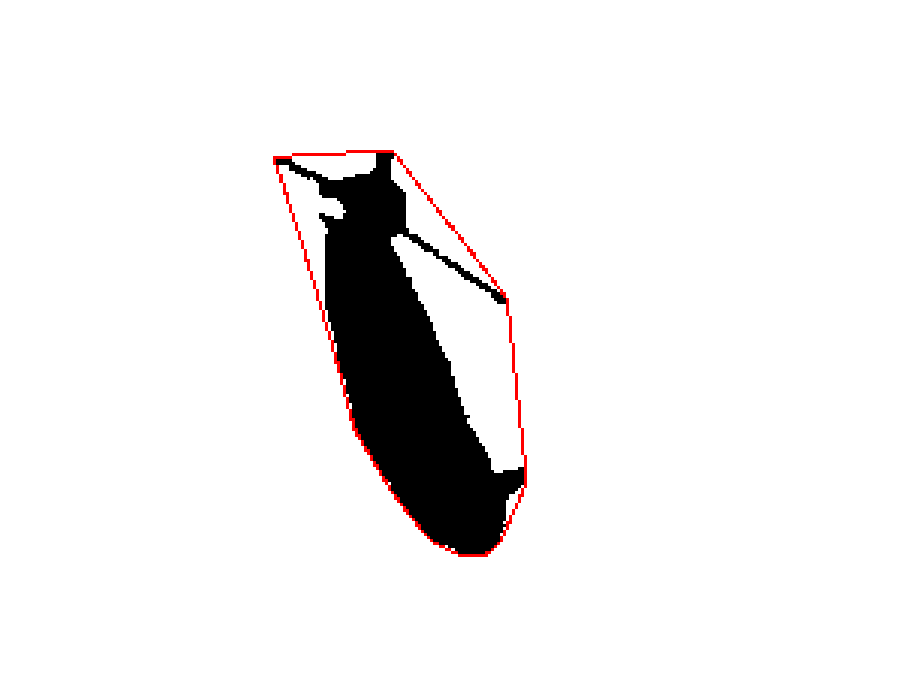}}\quad 
    \subfloat[helicopter]{\includegraphics[width=2.5cm]{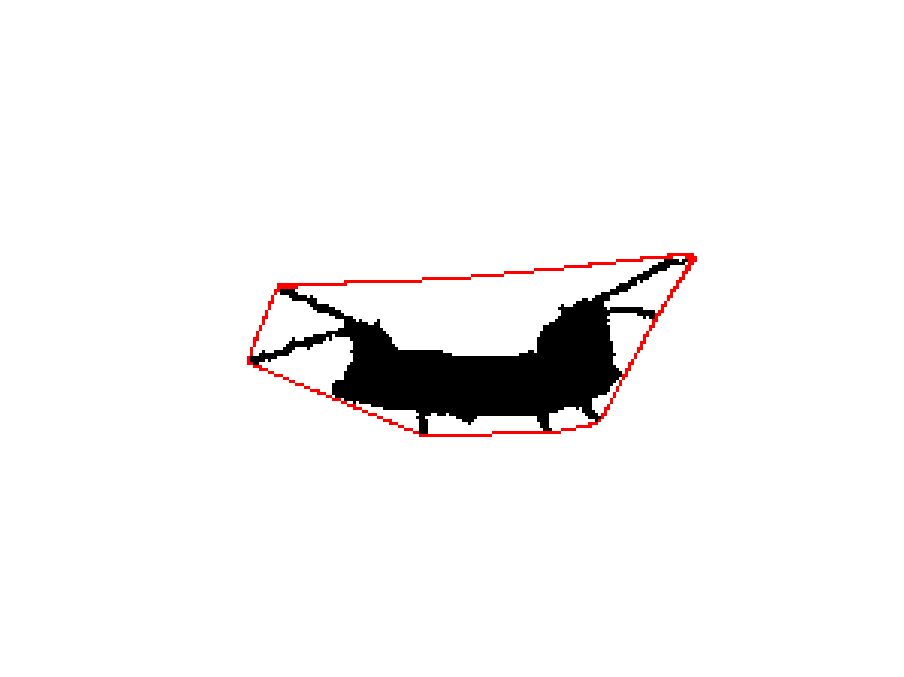}}\quad 
    \subfloat[Tendrils]{\includegraphics[width=2.5cm]{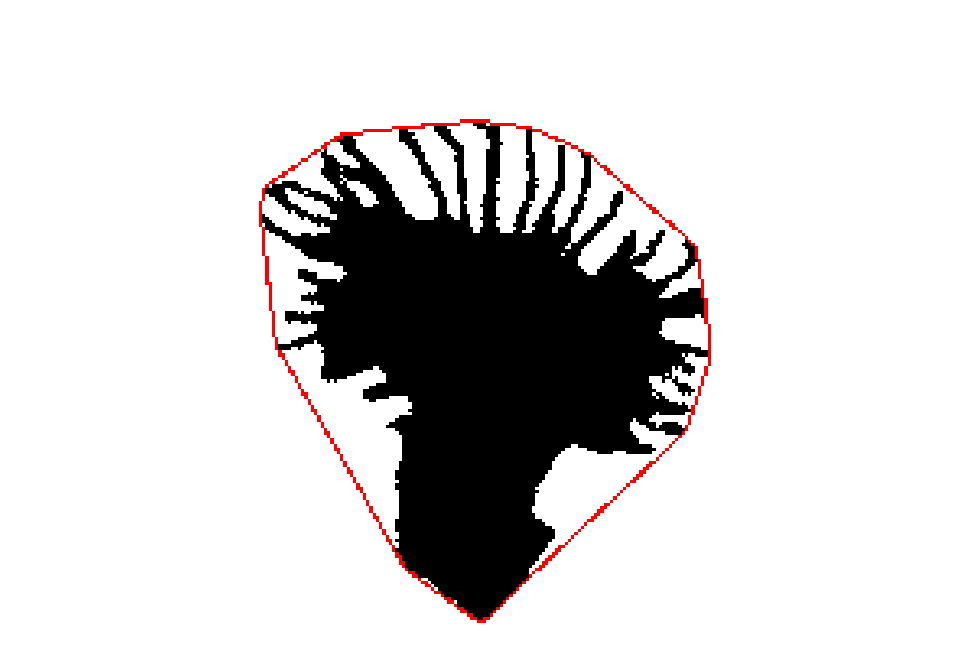}}
    \caption{Convex hulls of some 2-dimensional objects. The first row contains three images from \cite{AlpertGBB07} and the second row contains the corresponding ground true segmentation results. In the last row, we show the convex hulls of each object obtained by our proposed exact model.}
    \label{fig:2d_convhulls}
\end{figure}

\begin{table}
\caption{The relative errors compared with the quickhull result}
\centering
\begin{tabular}{@{}cccc@{}}
\toprule
method  & \textbf{boat}  & \textbf{Helicopter} & \textbf{Tendrils} \\ \midrule
our method & 1.08\%  & 1.13\%  & 1.03\% \\ 
Li et al.'s method \cite{li2019convex}  & 1.08\% & 1.13\% &0.73\%                 \\ \bottomrule
\end{tabular}
\label{table:error}
\end{table}

Then we turn our interest to 3-dimensional cases. We first conducted several experiments on some 3-dimensional shapes from the ShapeNet dataset \cite{shapenet2015}. We choose a chair object and a table object, and then compute their convex hulls using our algorithm. The original object and computed convex hulls are shown in Figure \ref{fig:chair} and \ref{fig:table}. For this set of experiments, we use $\rho_2=2000$, $\rho_3=10$, $\rho_1=2\sqrt{\rho_2\rho_3}$ and $\epsilon=10$. We also plot some level-set surfaces of the computed SDF in Figure \ref{fig:isosufaces}. For all the level-set surfaces up to $10$, we can see that they are convex, since we require the Hessian matrix is PSD in $L_{10}(|\phi|)$. When we look at the 15-level-set surfaces, we can easily find concavity. Similar to the 2D cases, we also compute the error compared to a benchmark result obtained by the  quickhull algorithm. The error is computed by (\ref{eq:error}) with $R$ defined as the radius of a 3D ball having the same volume with the benchmark convex hull. The error of the chair and table objects are $5.18\%$ and $4.31\%$ respectively.

\begin{figure}
    \centering
    \subfloat[]{\includegraphics[width=2.5cm]{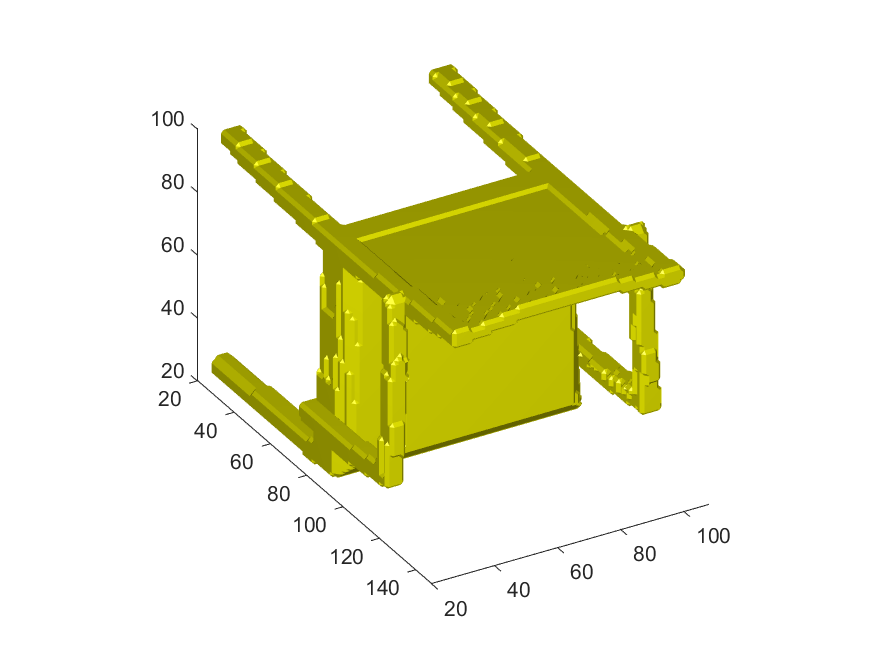}}\quad 
    \subfloat[]{\includegraphics[width=2.5cm]{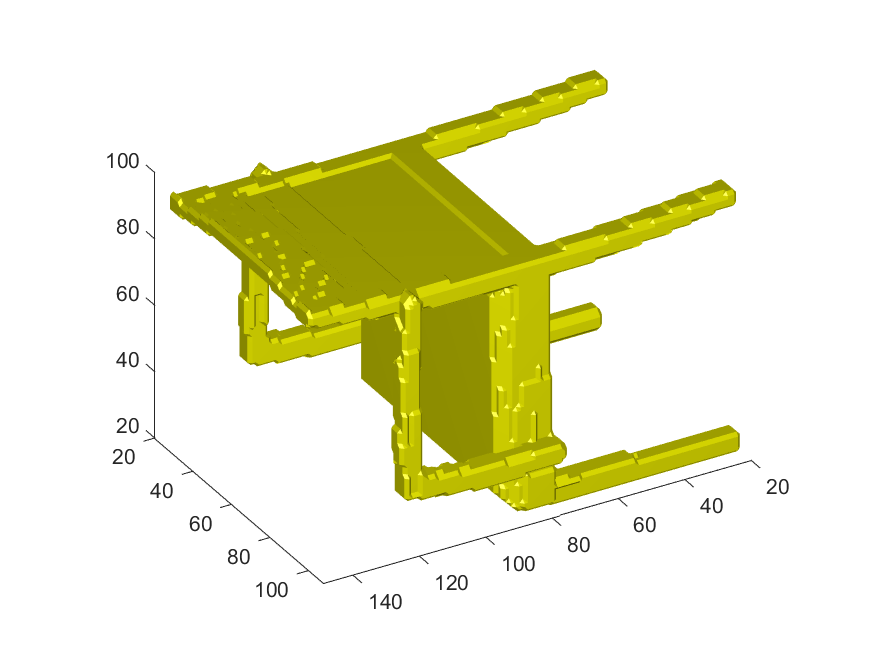}}\quad 
    \subfloat[]{\includegraphics[width=2.5cm]{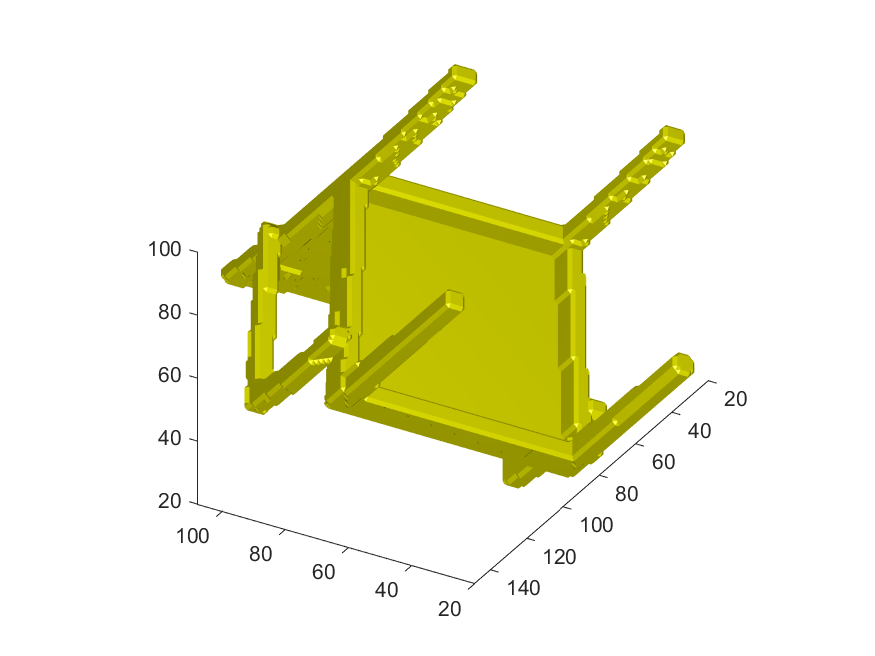}}\\
    \subfloat[]{\includegraphics[width=2.5cm]{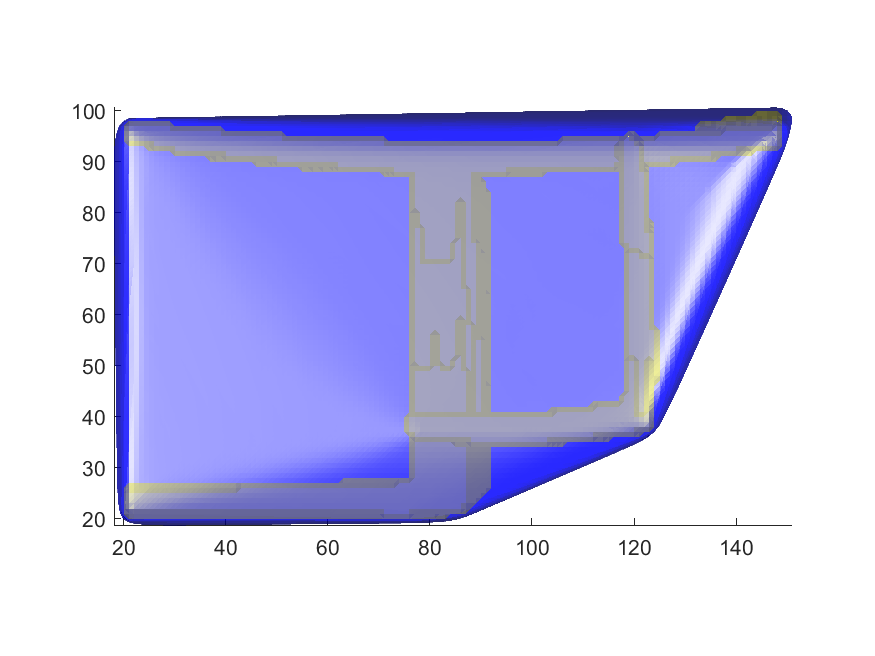}}\quad 
    \subfloat[]{\includegraphics[width=2.5cm]{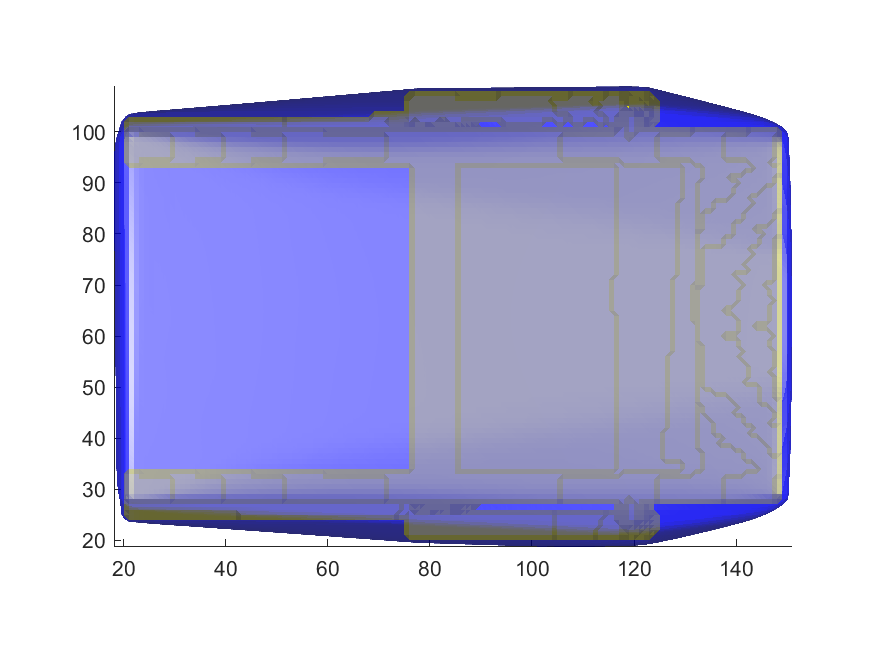}}\quad 
    \subfloat[]{\includegraphics[width=2.5cm]{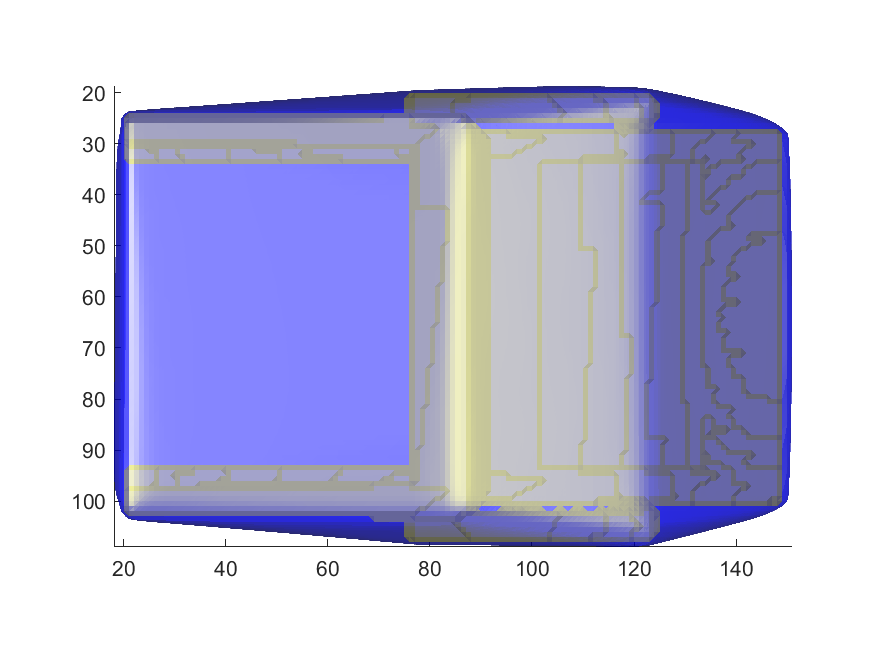}}\\
    \subfloat[]{\includegraphics[width=2.5cm]{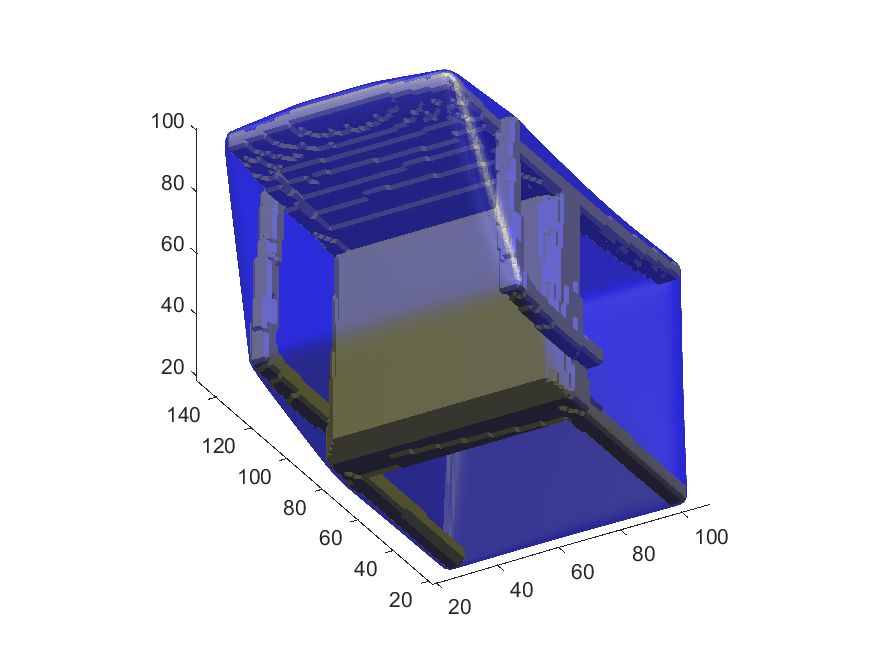}}\quad
    \subfloat[]{\includegraphics[width=2.5cm]{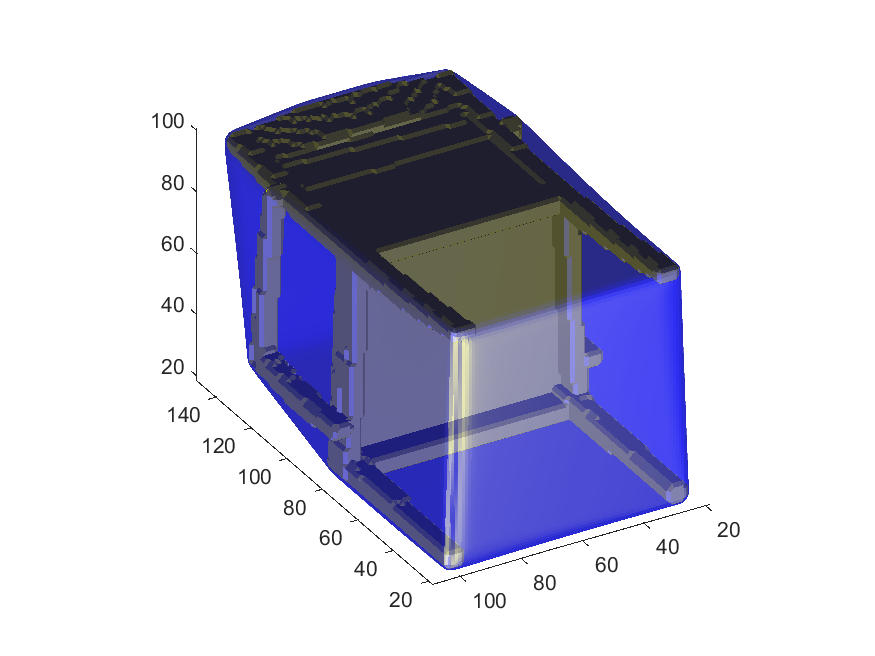}}\quad 
    \subfloat[]{\includegraphics[width=2.5cm]{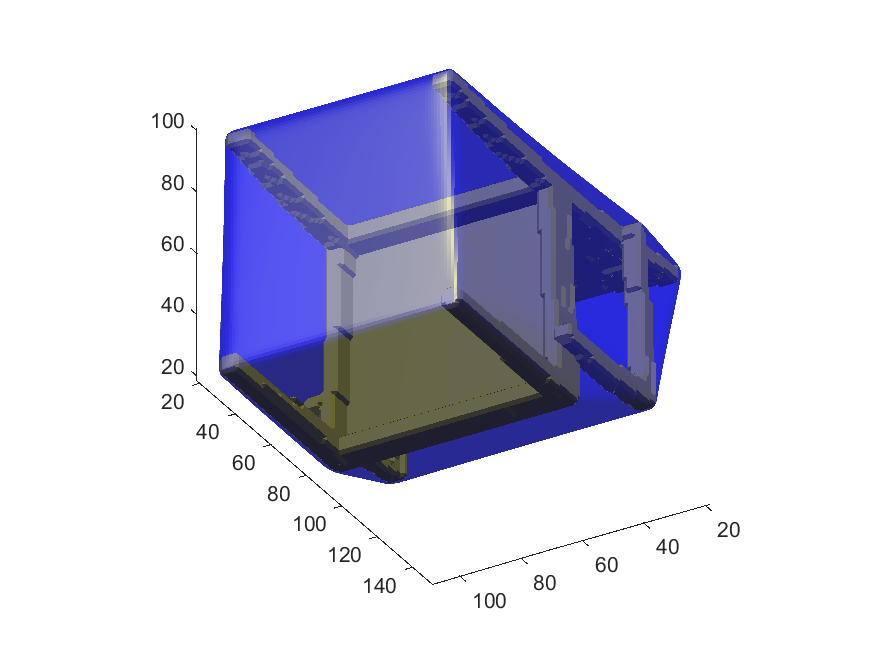}}
    \caption{Convex hull of a chair. The first row shows the original volumetric data and the last two rows show different views of the convex hull.}
    \label{fig:chair}
\end{figure}

\begin{figure}
    \centering
    \subfloat[]{\includegraphics[width=2.5cm]{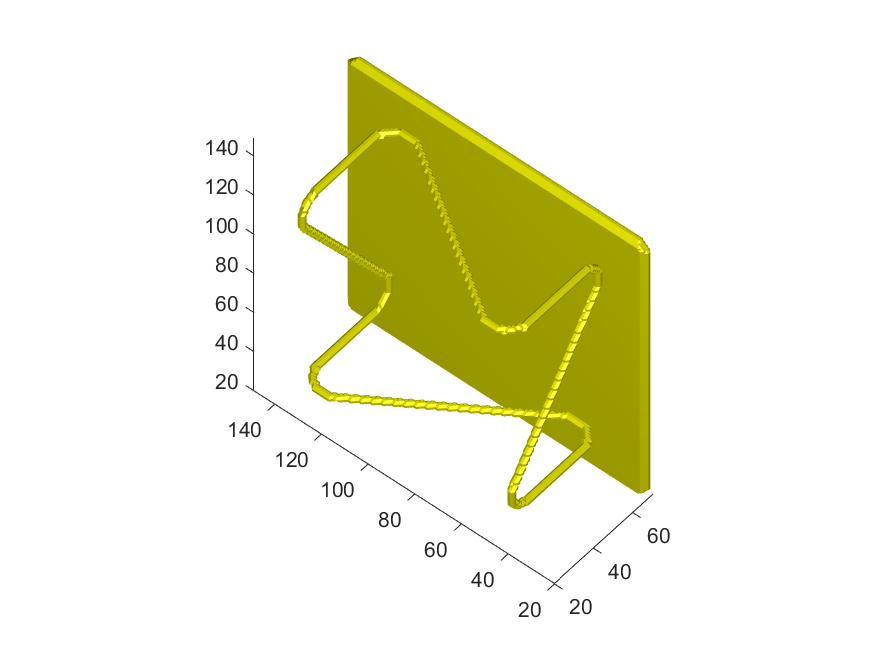}}\quad 
    \subfloat[]{\includegraphics[width=2.5cm]{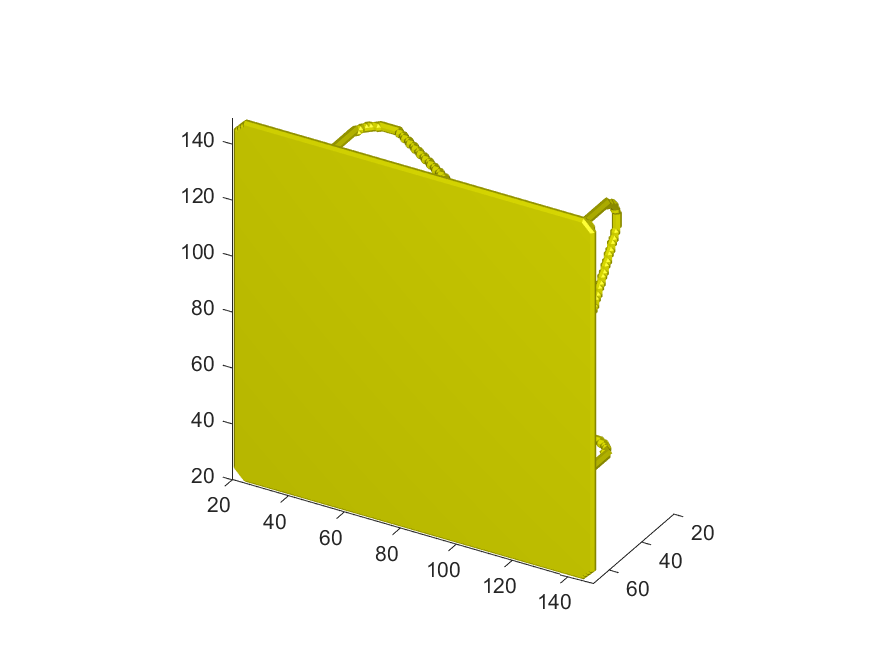}}\quad 
    \subfloat[]{\includegraphics[width=2.5cm]{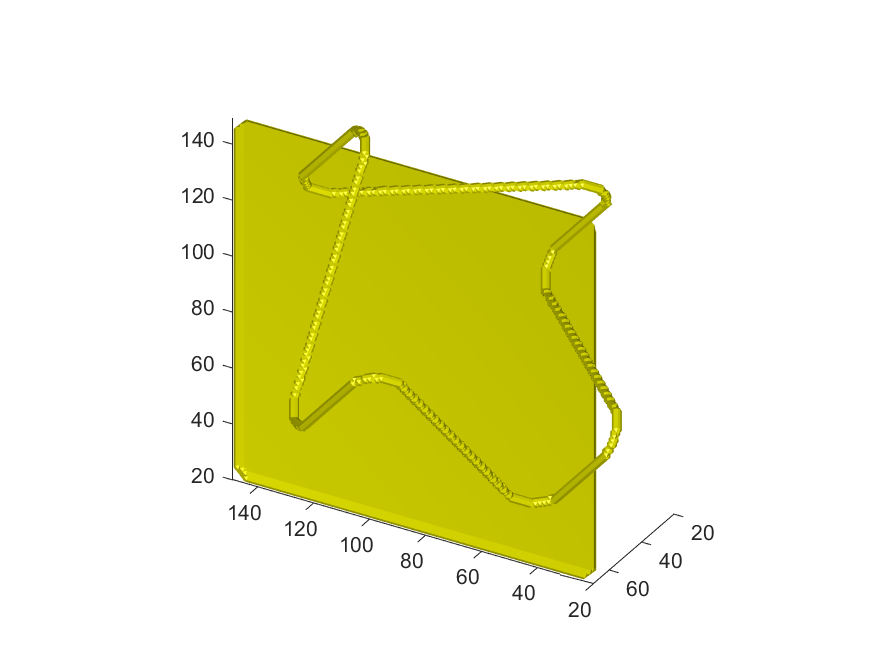}}\\
    \subfloat[]{\includegraphics[width=2.5cm]{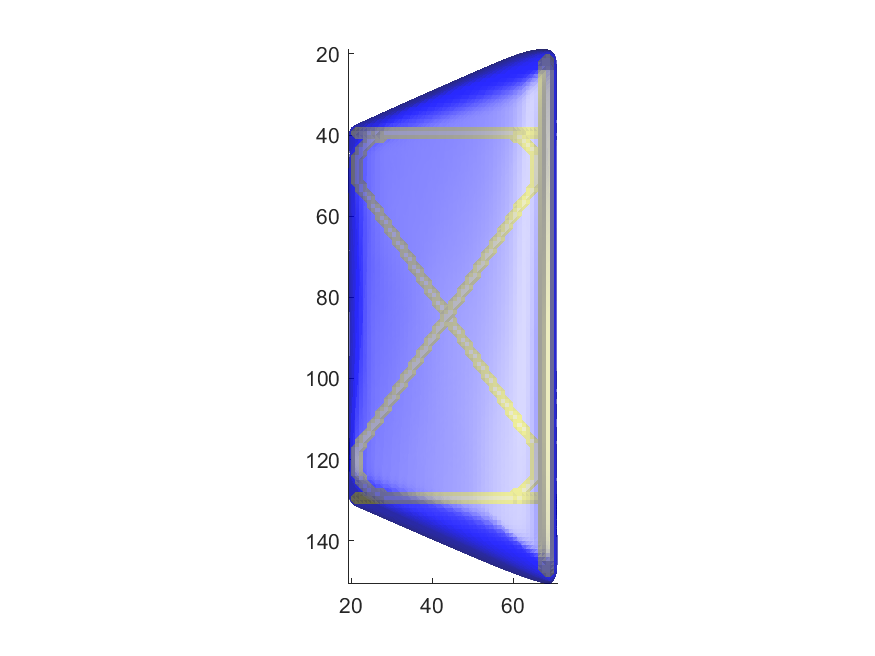}}\quad 
    \subfloat[]{\includegraphics[width=2.5cm]{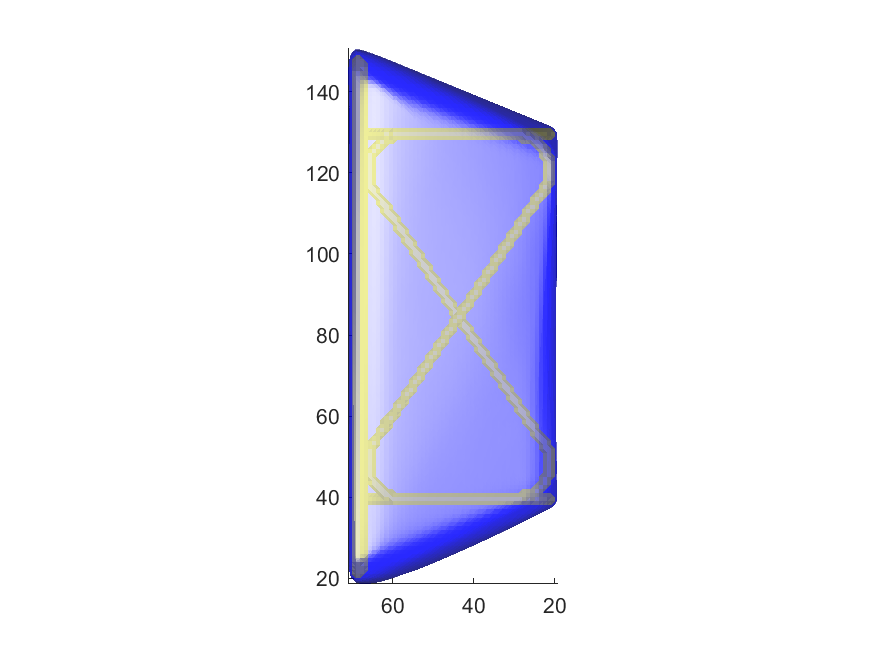}}\quad 
    \subfloat[]{\includegraphics[width=2.5cm]{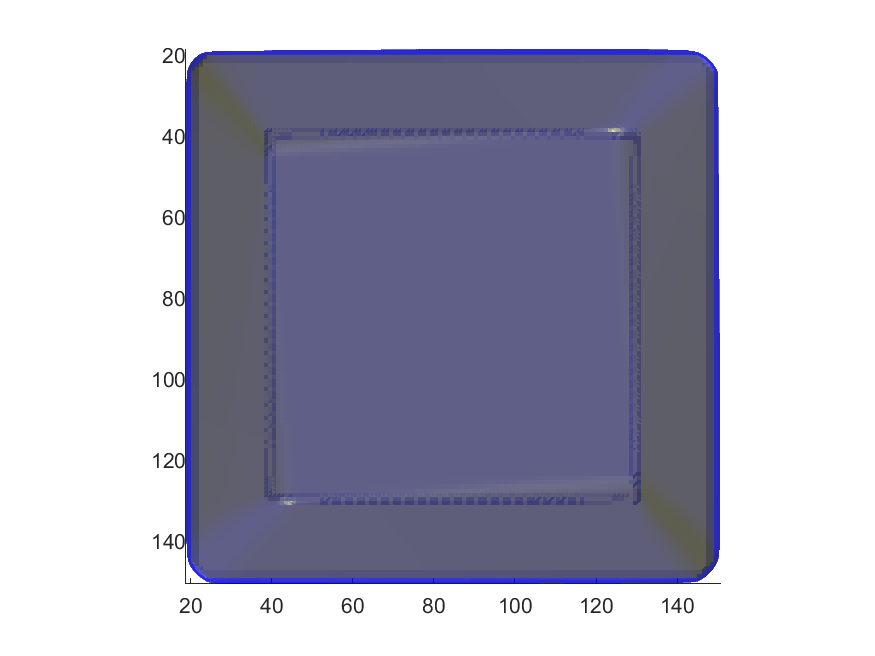}}\\
    \subfloat[]{\includegraphics[width=2.5cm]{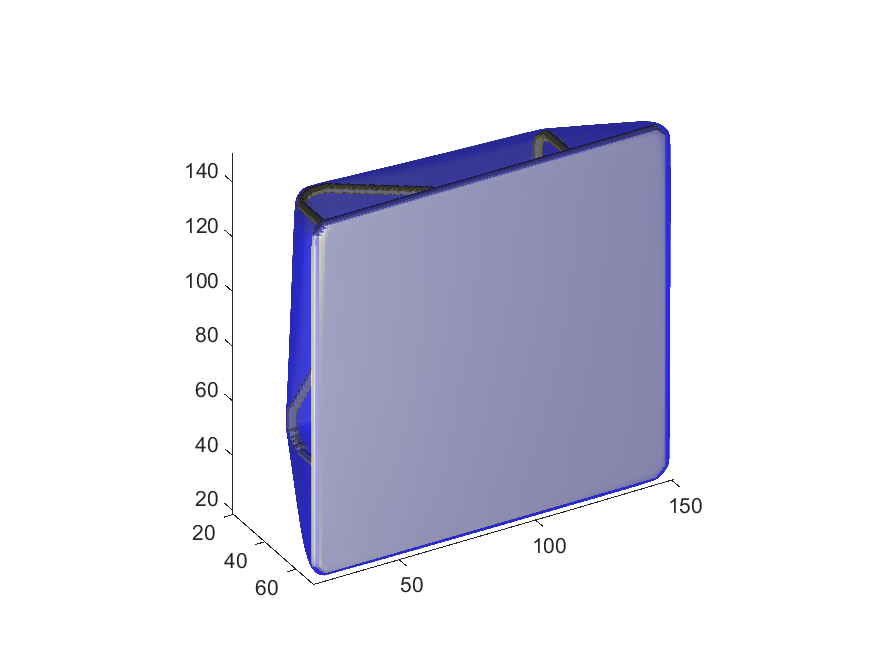}}\quad
    \subfloat[]{\includegraphics[width=2.5cm]{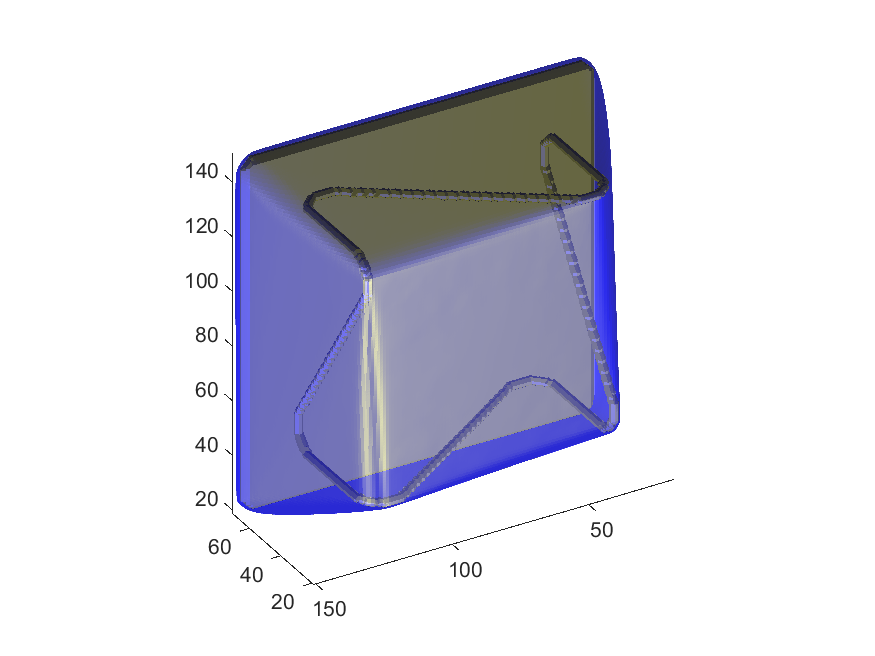}}\quad 
    \subfloat[]{\includegraphics[width=2.5cm]{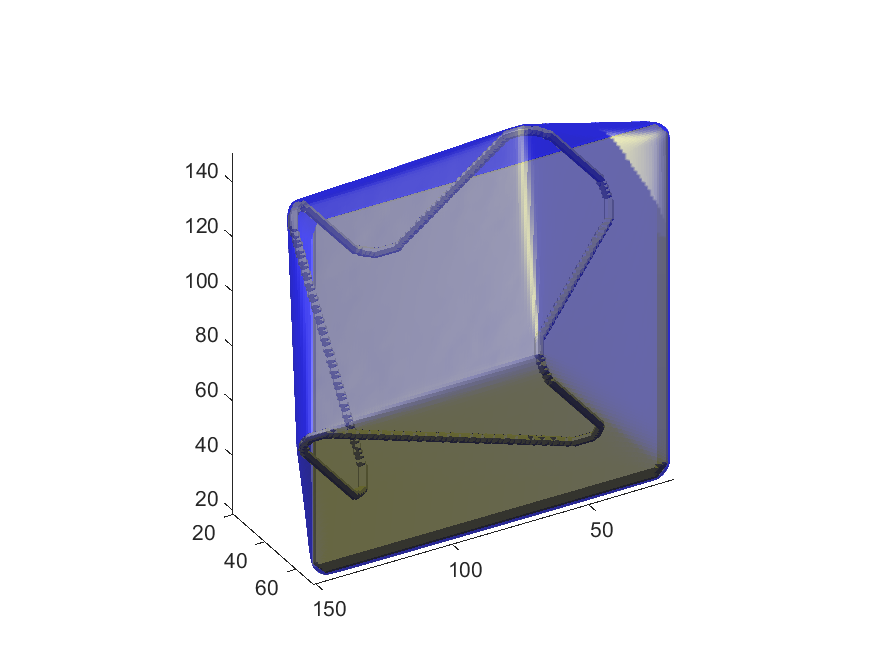}}
    \caption{Convex hull of a table. The first row shows the original volumetric data and the last two rows show different views of the convex hull.}
    \label{fig:table}
\end{figure}

% \begin{figure}
%     \centering
%     \subfloat[]{\includegraphics[width=2.5cm]{figures/beetle/stagbeetle_v1.png}}\quad 
%     \subfloat[]{\includegraphics[width=2.5cm]{figures/beetle/stagbeetle_v2.png}}\quad 
%     \subfloat[]{\includegraphics[width=2.5cm]{figures/beetle/stagbeetle_v3.png}}\\
%     \subfloat[]{\includegraphics[width=2.5cm]{figures/beetle/stagbeetle_v4(2000,10.png}}\quad
%     \subfloat[]{\includegraphics[width=2cm]{figures/beetle/stagbeetle_v5(2000,10).png}}\quad 
%     \subfloat[]{\includegraphics[width=2.5cm]{figures/beetle/stagbeetle_v6(2000,10).png}}\\
%     \subfloat[]{\includegraphics[width=2.5cm]{figures/beetle/stagbeetle_v1(2000,10).png}}\quad 
%     \subfloat[]{\includegraphics[width=2.5cm]{figures/beetle/stagbeetle_v2(2000,10).png}}\quad 
%     \subfloat[]{\includegraphics[width=2.5cm]{figures/beetle/stagbeetle_v3(2000,10).png}}
%     \caption{Convex hull of a stag-beetle. The first row shows the original volumetric data and the last two rows show different views of the convex hull.}
%     \label{fig:stagbeetle}
% \end{figure}

\begin{figure}[p]
    \centering
    \subfloat[]{\includegraphics[width=3cm]{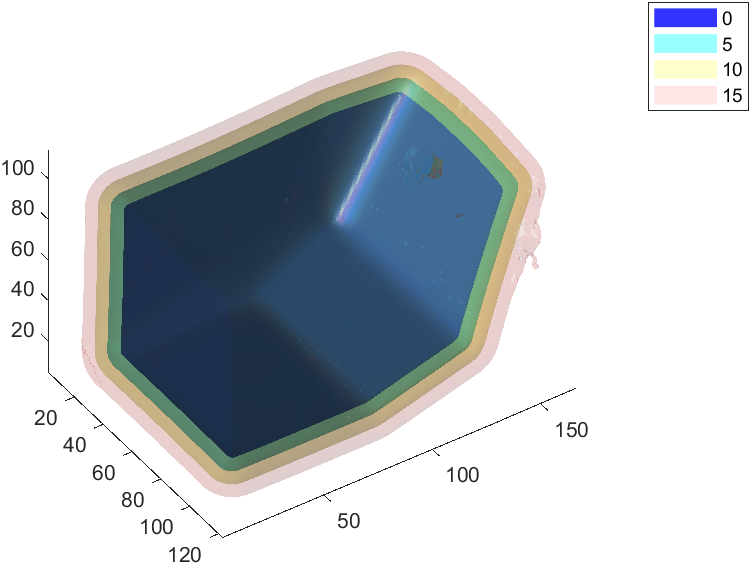}}\quad 
    \subfloat[]{\includegraphics[width=3cm]{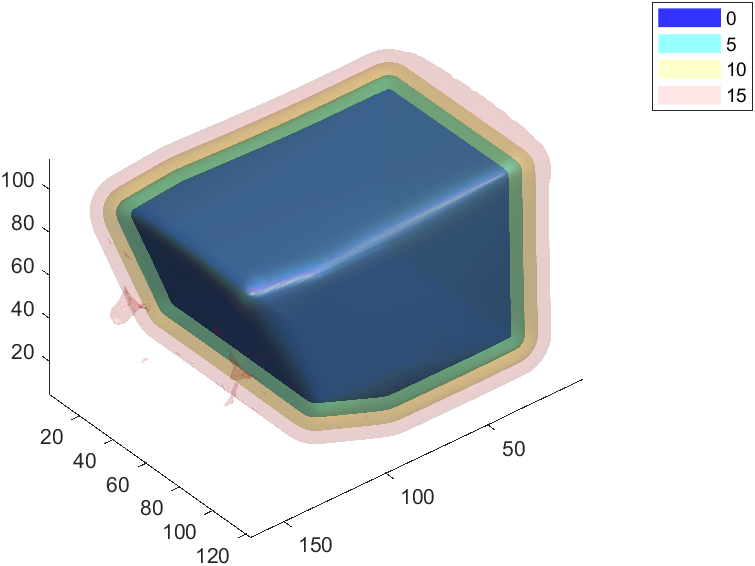}}\quad 
    \subfloat[]{\includegraphics[width=3cm]{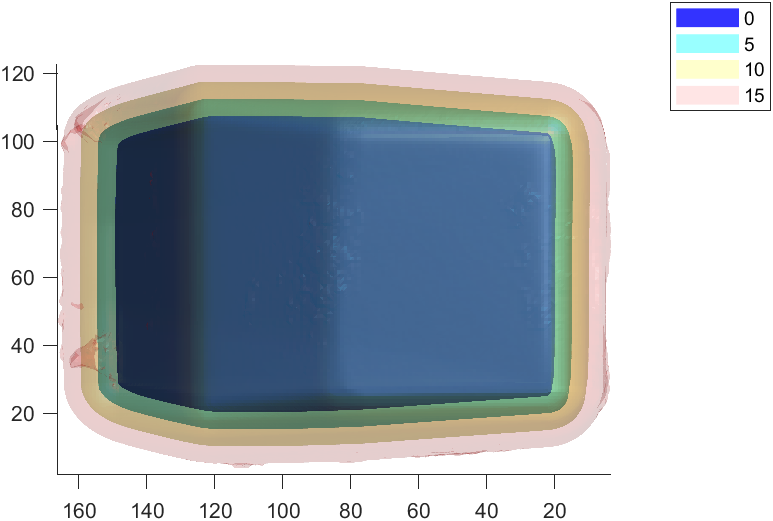}}\\
    \subfloat[]{\includegraphics[width=3cm]{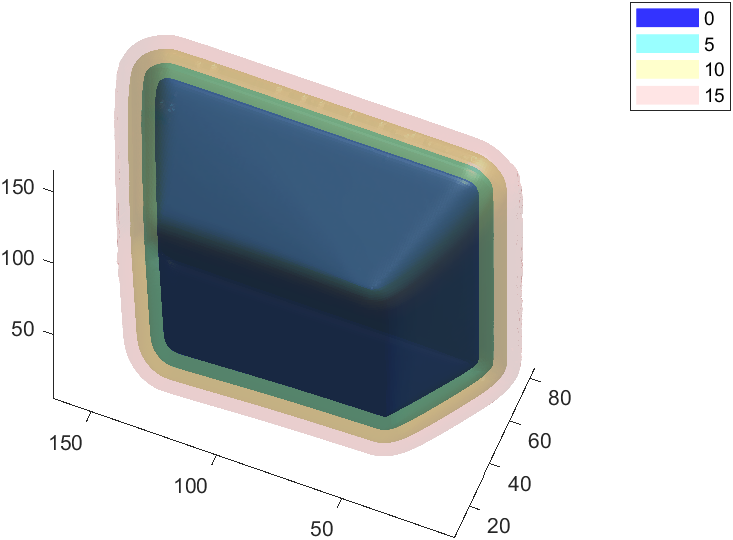}}\quad 
    \subfloat[]{\includegraphics[width=3cm]{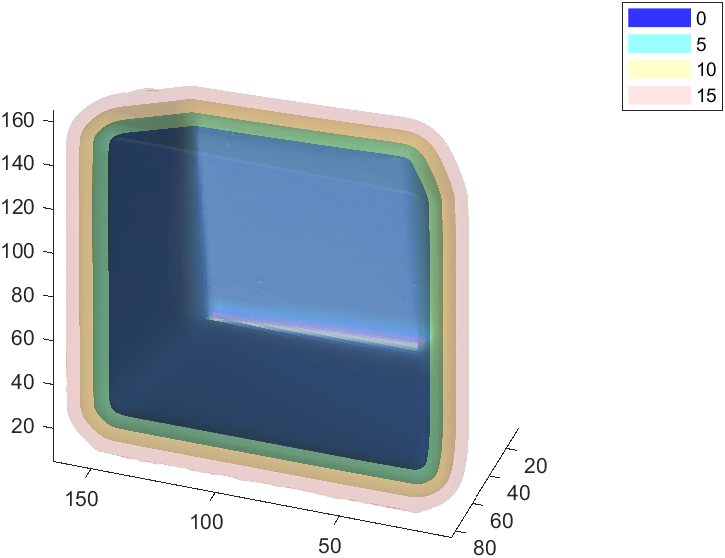}}\quad 
    \subfloat[]{\includegraphics[width=3cm]{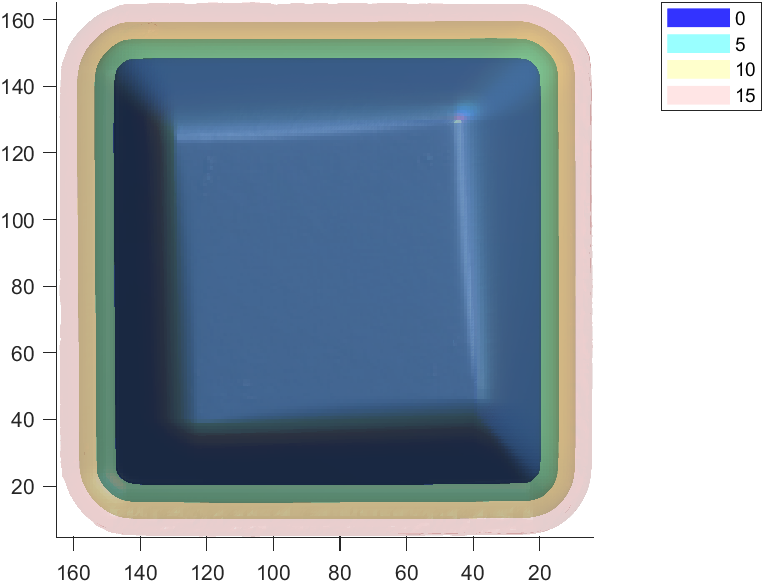}}
    \caption{Different level-set surfaces. Here we plot the $0$, $5$, $10$ and $15$ level-set surfaces of the computed SDFs. The first row is for the chair object and the second row is for the table object.}
    \label{fig:isosufaces}
\end{figure}

%  chair1 5.18% 4.09% table1 4.31% 6.95%

Then we conduct an experiment on a volume with two cars to show that our method is able to handle multiple objects. When there are more than one object in the given volume, we may be interested in obtaining convex hulls for each object separately. If we use some conventional methods to do this, we may need some object detection algorithm to extract the region of each object first. However, in our algorithms, we can achieve this by selecting a small $\epsilon$ value. Recall that the convexity constraint we imposed is that $H(\phi)\geq 0$ in $L_\epsilon(|\phi|)$. The SDF of multiple convex hulls is convex only in a small neighbourhood around each object. As long as our $\epsilon$ is small enough, our algorithm will return the separated convex hulls of each object. More specifically, $\epsilon$ should be smaller than half of the distance between each pair of objects. The numerical results are shown in Figure \ref{fig:twocars_convexhulls}. We use the same set of parameters with the previous experiment here. From the result, we can see that when $\epsilon=5$, our exact algorithm can compute the separated convex hulls accurately. When we set $\epsilon=20$, we can get the big convex hull containing two cars together. We also plot the level-set surfaces to further illustrate how it works in Figure \ref{fig:twocars_levelset}. When $\epsilon=5$, we plot the 5 and 20 level-set surfaces. We can observe that the 5 level-set consists of two separated surfaces and both of them are convex. If we look at the 20 level-set surface, it is non-convex at somewhere between two cars. Since we only require $H(\phi)\geq 0$ in $L_5(|\phi|)$, the non-convexity on the 20 level-set does not violate the constraint. However, if we set $\epsilon=20$, this SDF is not feasible any more and what we will get is the SDF in (b) which corresponds to the convex hull of two cars combined.

\begin{figure}
    \centering
    \subfloat[original data]{\includegraphics[width=2.5cm]{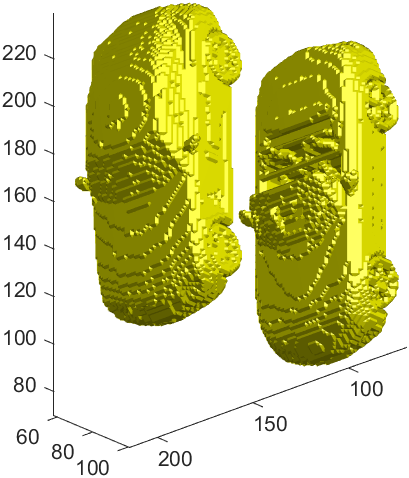}}\quad 
    \subfloat[$\epsilon=5$]{\includegraphics[width=2.5cm]{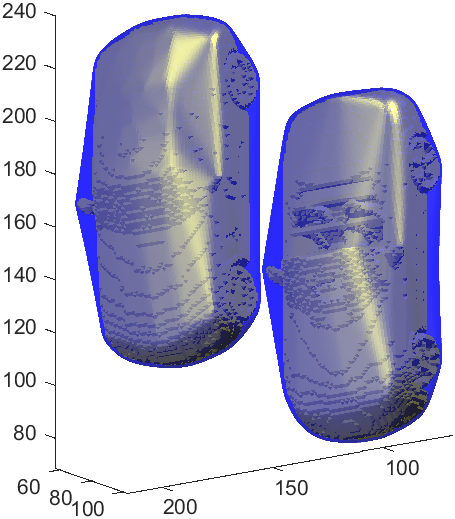}}\quad 
    \subfloat[$\epsilon=20$]{\includegraphics[width=2.5cm]{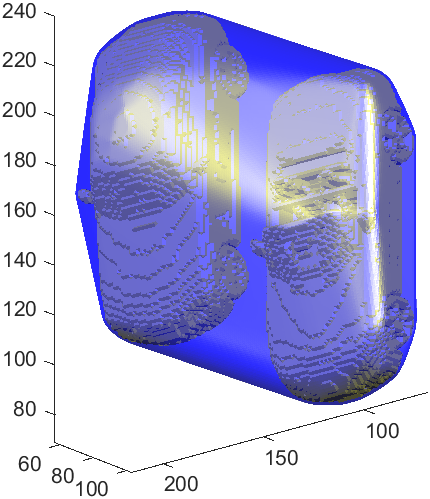}}\\
    \subfloat[original data]{\includegraphics[width=2.5cm]{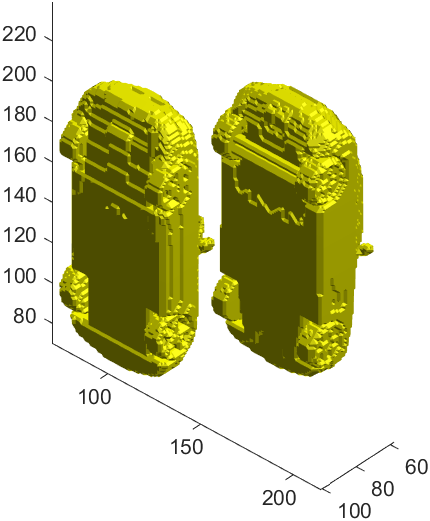}}\quad 
    \subfloat[$\epsilon=5$]{\includegraphics[width=2.5cm]{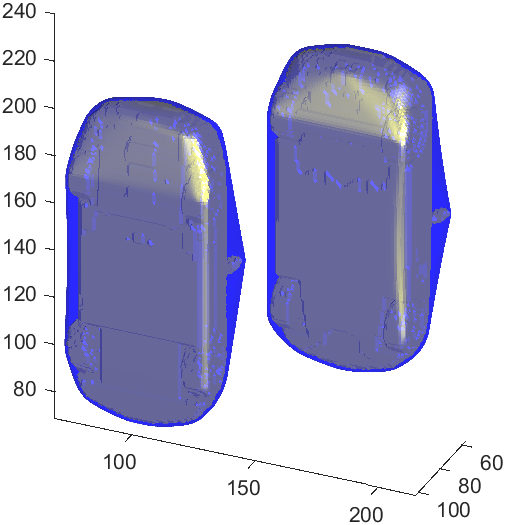}}\quad 
    \subfloat[$\epsilon=20$]{\includegraphics[width=2.5cm]{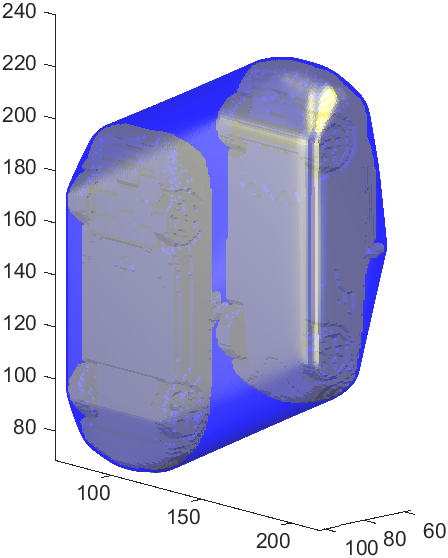}}\\
    \subfloat[original data]{\includegraphics[width=2.5cm]{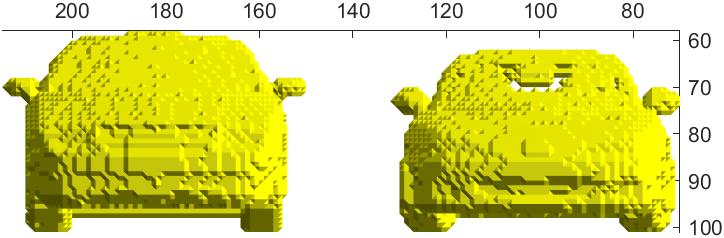}}\quad
    \subfloat[$\epsilon=5$]{\includegraphics[width=2.5cm]{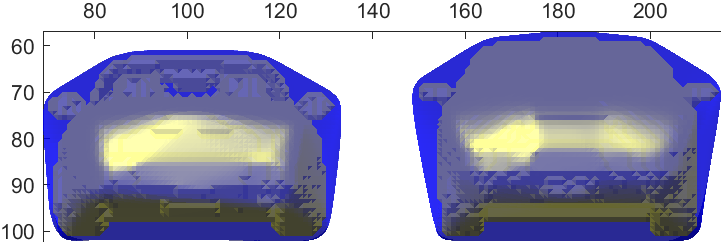}}\quad 
    \subfloat[$\epsilon=20$]{\includegraphics[width=2.5cm]{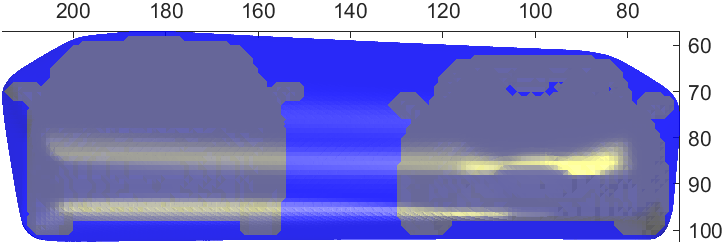}}
    \caption{Convex hulls of multiple objects. The first column shows the original data of two cars. The second column is the convex hulls when setting $\epsilon=5$ and the third column is the convex hull when setting $\epsilon=20$.}
    \label{fig:twocars_convexhulls}
\end{figure}

\begin{figure}
    \centering
    \subfloat[$\epsilon=5$]{\includegraphics[width=3cm]{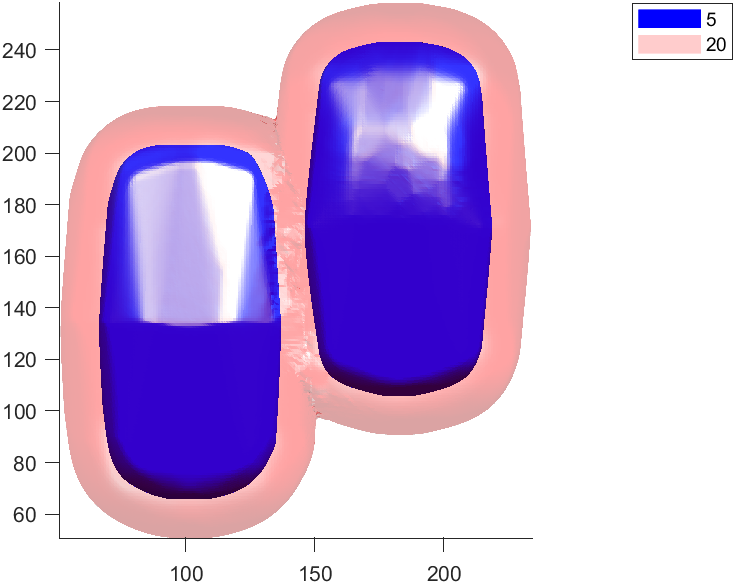}}\quad
    \subfloat[$\epsilon=20$]{\includegraphics[width=3cm]{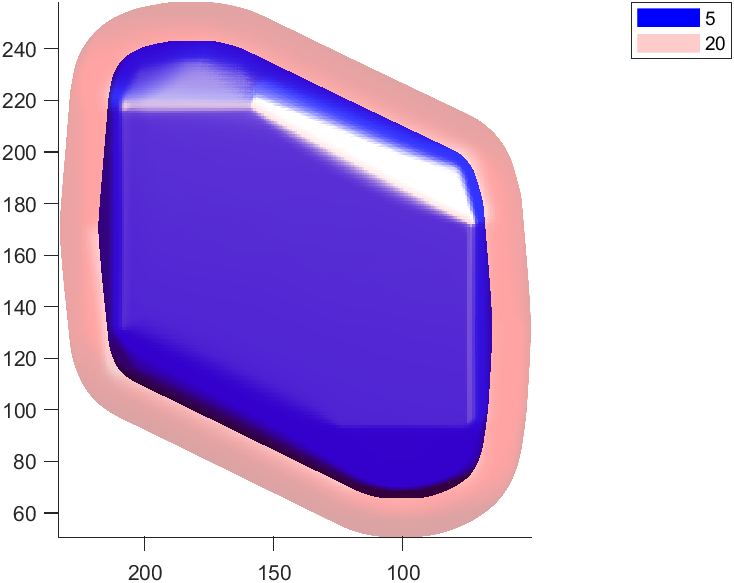}}
    \caption{Level-set surfaces of the corresponding SDF when $\epsilon=5$ and $\epsilon=20$.}
    \label{fig:twocars_levelset}
\end{figure}

As demonstrated in \cite{li2019convex}, the variational convex hull algorithm is very robust to noise and outliers. In 3-dimensional cases, the model can still preserve this advantage after applying our proposed prior. We choose two objects from the shape-net dataset \cite{shapenet2015} and generate some outliers, which are randomly sampled from a uniform distribution, to the volumetric data. The approximate results are shown in Figure \ref{fig:camera} and \ref{fig:headphone}. The parameters we used in this experiments is $\rho0=400$, $\rho_2=2000$, $\rho_1=2\sqrt{\rho_0\rho_2}$, $\lambda=800$ and $\epsilon=5$. 
\begin{figure}
    \centering
    \subfloat[original data]{\includegraphics[width=2cm]{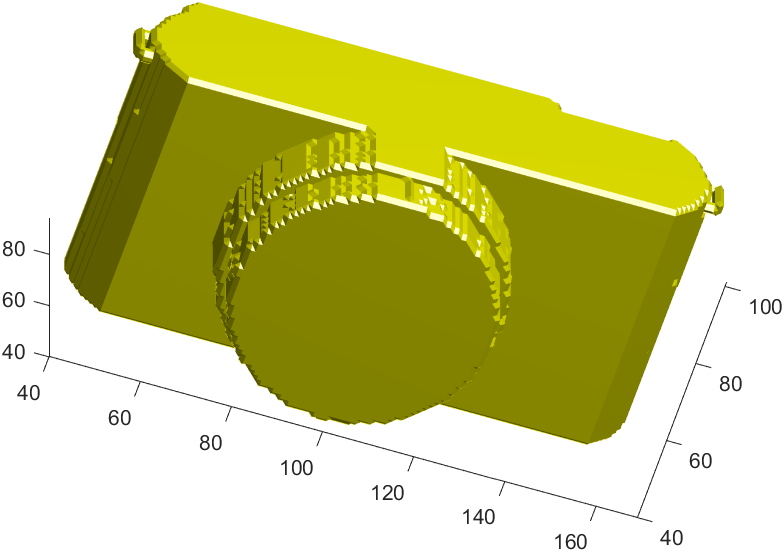}}\quad 
    \subfloat[original data]{\includegraphics[width=2cm]{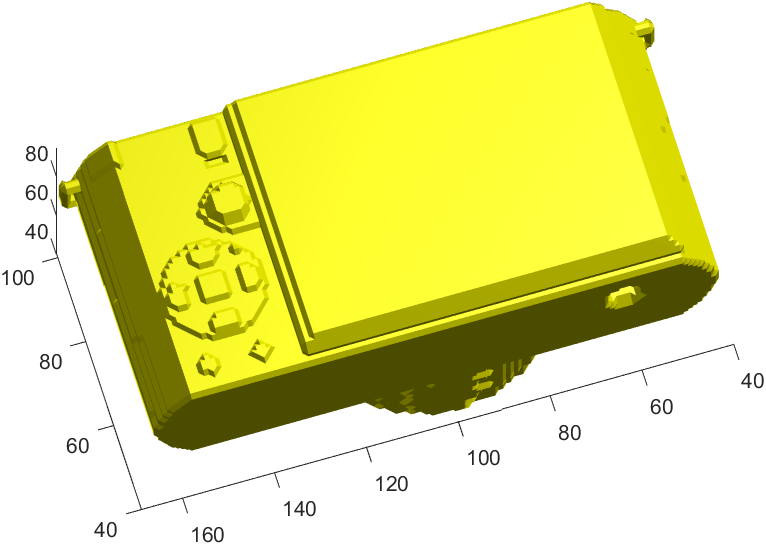}}\quad 
    \subfloat[generated outliers]{\includegraphics[width=2cm]{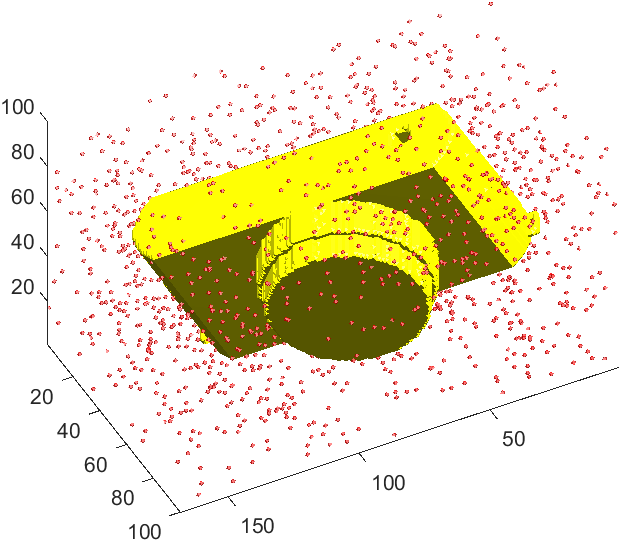}}\\
    \subfloat[convex hull]{\includegraphics[width=2.5cm]{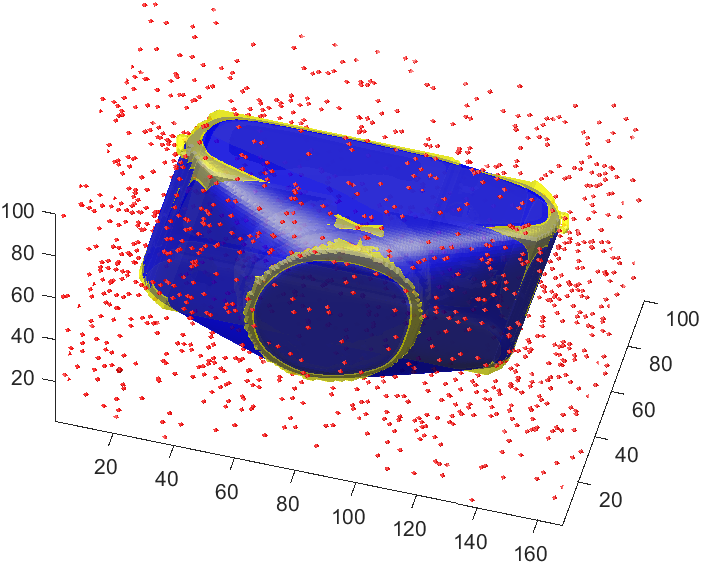}}\quad 
    \subfloat[convex hull]{\includegraphics[width=2.5cm]{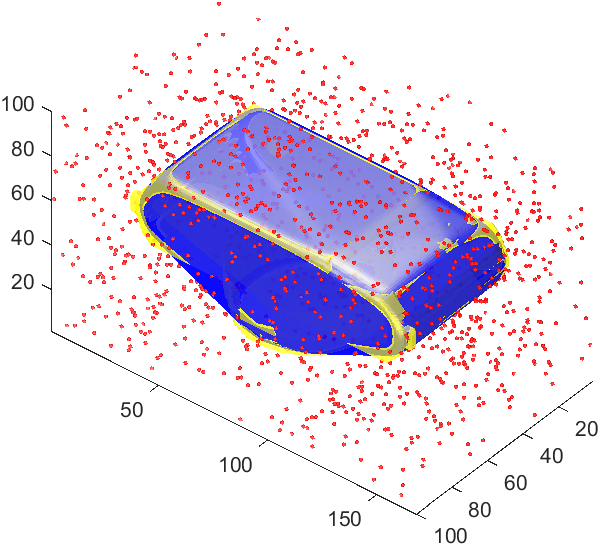}}\quad 
    \subfloat[convex hull]{\includegraphics[width=2.5cm]{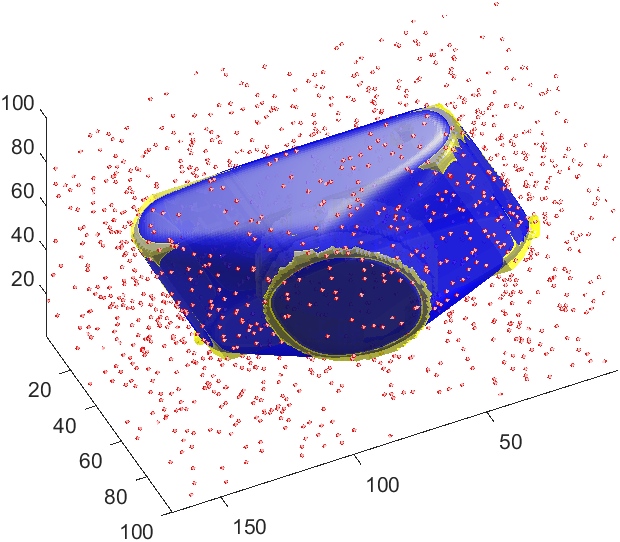}}\\
    \subfloat[convex hull]{\includegraphics[width=2.5cm]{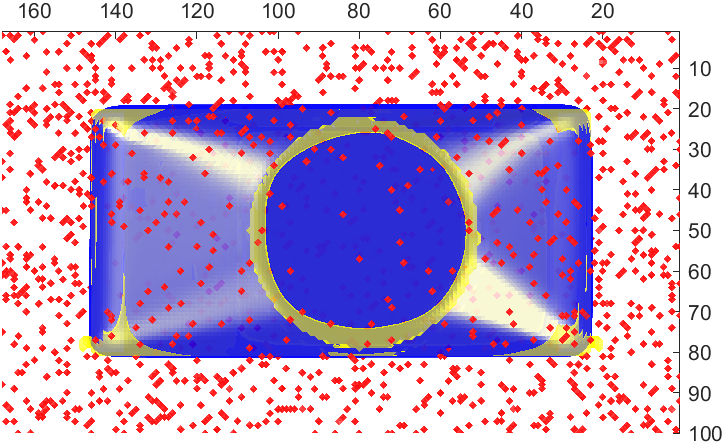}}\quad
    \subfloat[convex hull]{\includegraphics[width=2.5cm]{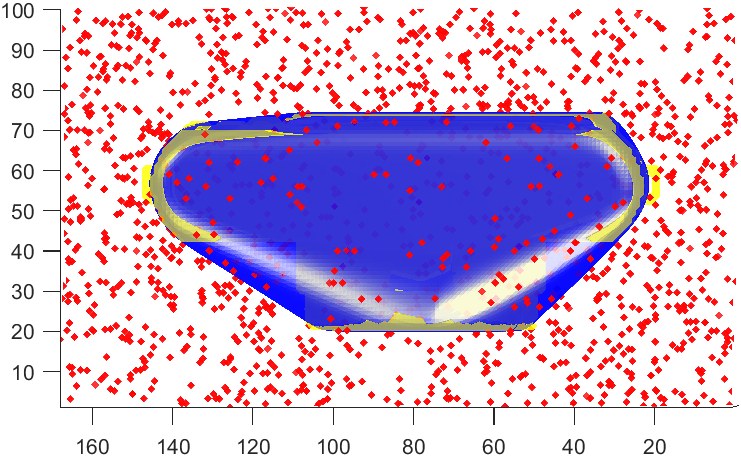}}\quad 
    \subfloat[convex hull]{\includegraphics[width=1.7cm]{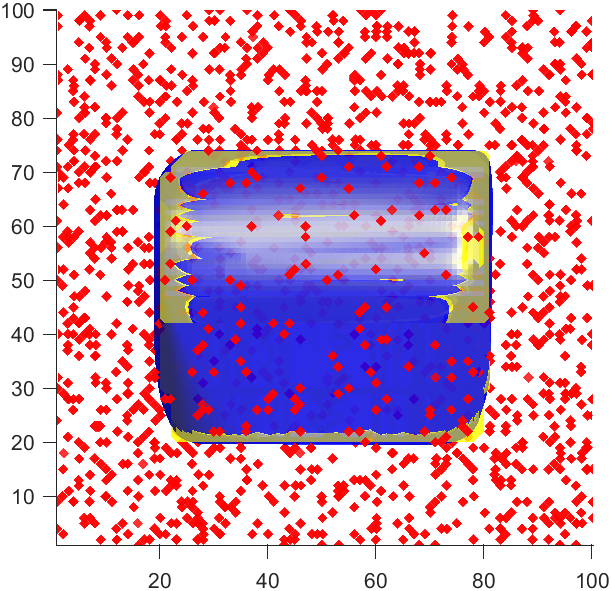}}
    \caption{Convex hull of a camera with some outliers.}
    \label{fig:camera}
\end{figure}

% cl 33.29% ours 8.45%

\begin{figure}
    \centering
    \subfloat[original data]{\includegraphics[width=2cm]{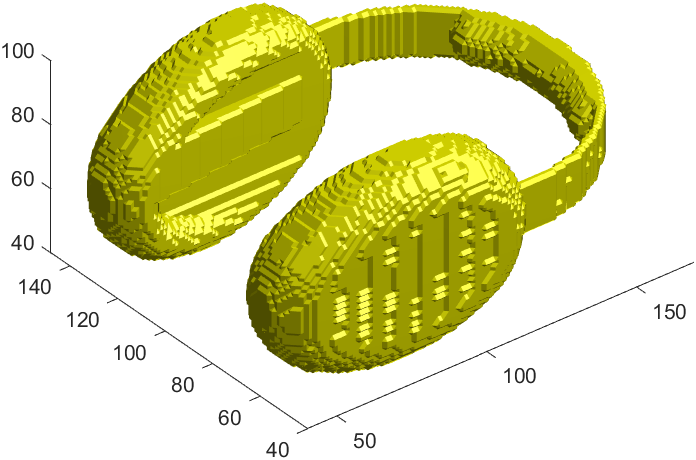}}\quad 
    \subfloat[original data]{\includegraphics[width=2cm]{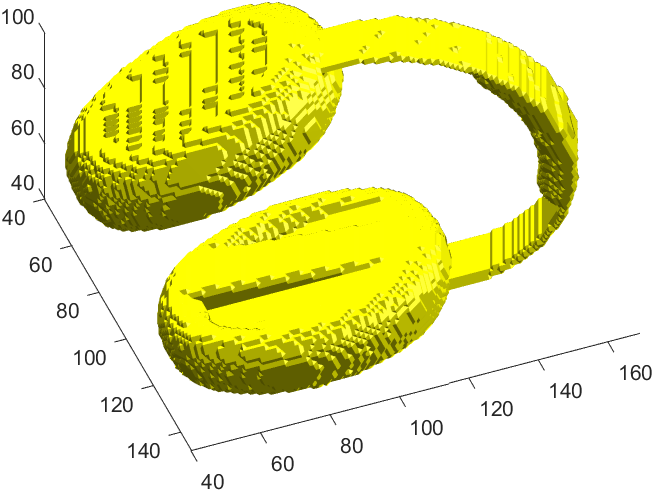}}\quad 
    \subfloat[generated outliers]{\includegraphics[width=2cm]{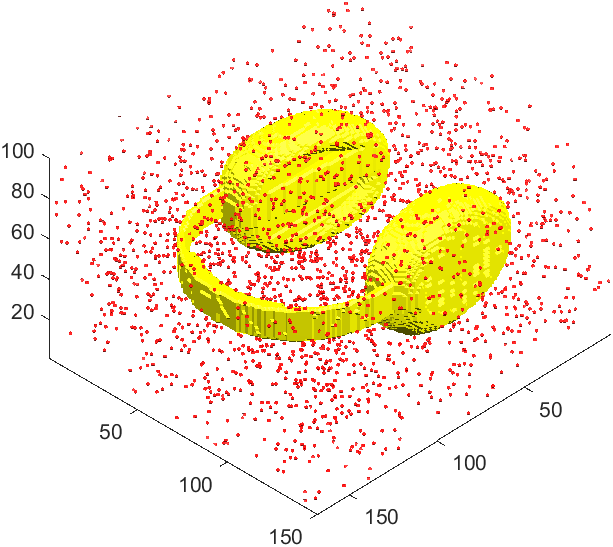}}\\
    \subfloat[convex hull]{\includegraphics[width=2.5cm]{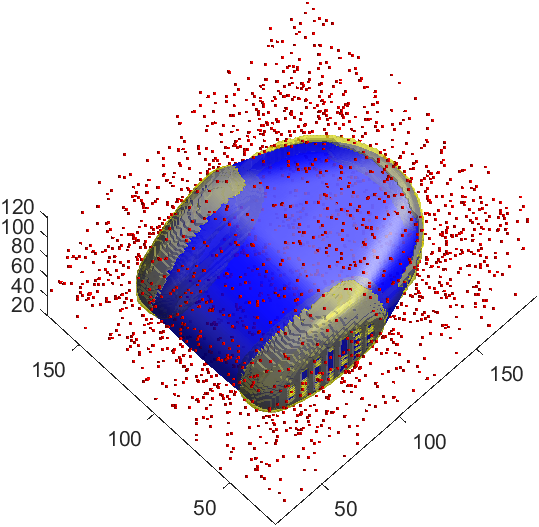}}\quad 
    \subfloat[convex hull]{\includegraphics[width=2.5cm]{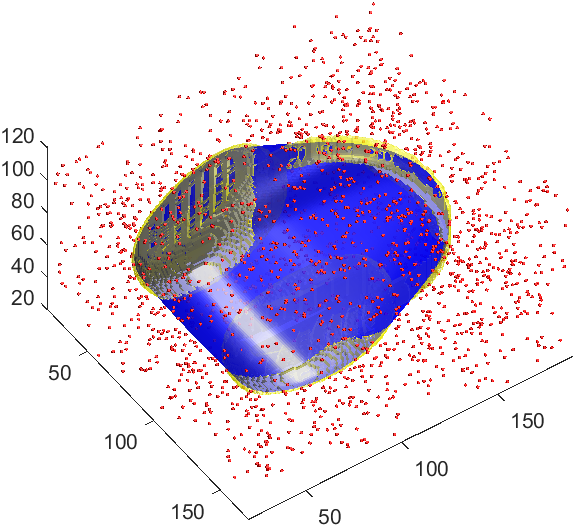}}\quad 
    \subfloat[convex hull]{\includegraphics[width=2.5cm]{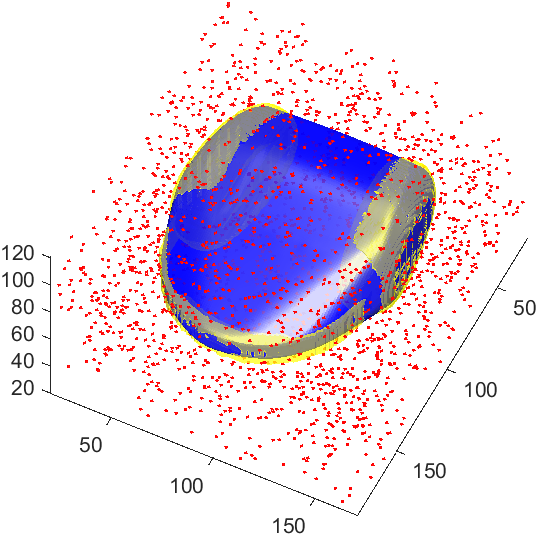}}\\
    \subfloat[convex hull]{\includegraphics[width=2.5cm]{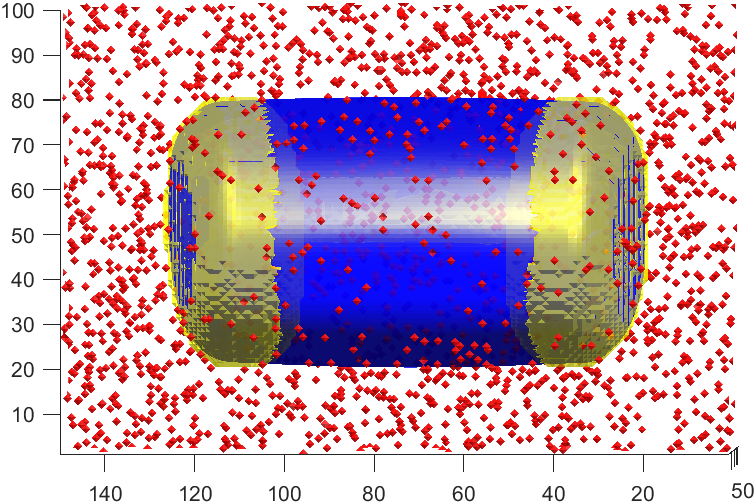}}\quad
    \subfloat[convex hull]{\includegraphics[width=2.5cm]{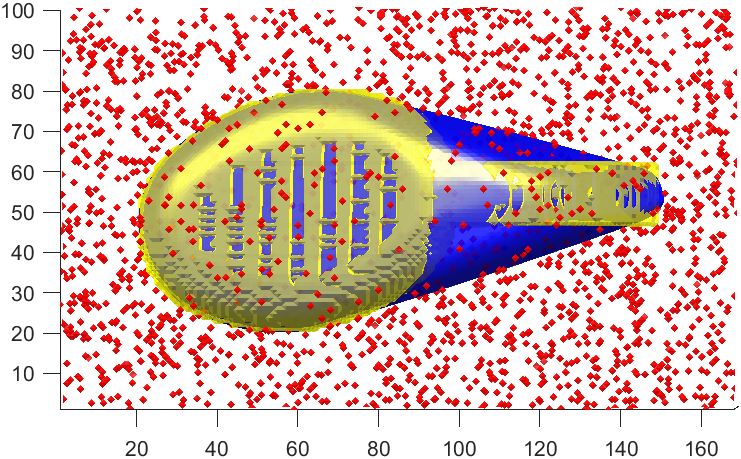}}\quad 
    \subfloat[convex hull]{\includegraphics[width=1.7cm]{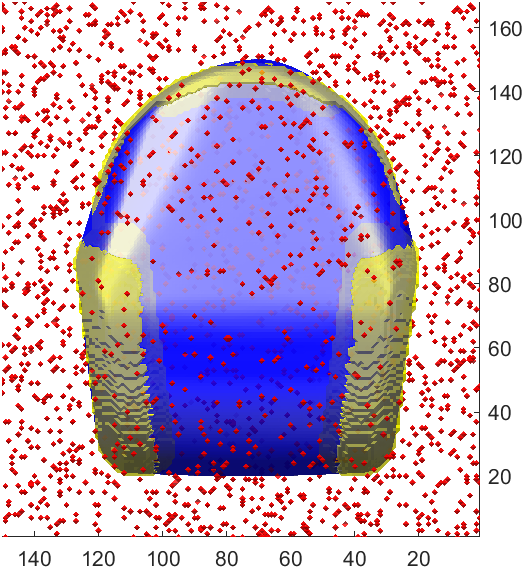}}
    \caption{Convex hull of a headphone with some outliers.}
    \label{fig:headphone}
\end{figure}
% cl 47.83%  ours 9.80%

We can observe that even if there exist large amount of outliers in the input, our algorithm can still obtain a very good approximation to the original convex hull. We also compare our result with the convex layer methods \cite{chazelle1985convex}, which is also called the onion peeling method. Given a finite set of points $X$, the convex hull of $X$ is called the first convex layer. Then, we remove those points on the boundary of Conv$(X)$ and compute the convex hull for the rest points. The new convex hull is called the second convex layer. By continuing the same procedure, we can get a set of convex layers and each point in $X$ must belong to one layer. The convex layer method for outliers detection usually relies on two assumptions \cite{harsh2018onion}. Firstly, the outliers are located in the first few convex layers. In other words, the outliers are evenly distributed around the object. Secondly, the approximate number of outliers is known, which is not easy to obtain in some situations. We briefly described a convex layer algorithm for convex hull approximation in Algorithm \ref{algo:convex_layers}. For the camera and the headphone objects, the estimated number  of outliers $K$ are set to be $1400$ and $2300$, and the exact number of outliers are $1371$ and $2268$ respectively. From the results of convex layers method in Figure \ref{fig:convex_layers}, we see that even a very accurate $k$ is given, the approximate convex hulls still deviate a lot from the true solution. Looking at the error in Table \ref{table:error2}, our method beats the convex layer method by a huge margin.

 \begin{figure}
    \centering
    \subfloat[$camera$]{\includegraphics[width=4cm]{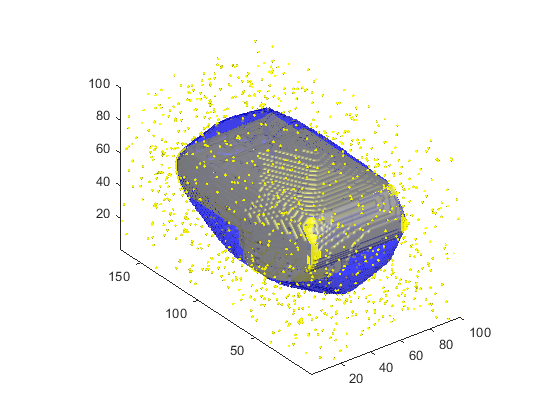}}\quad
    \subfloat[$headphone$]{\includegraphics[width=4cm]{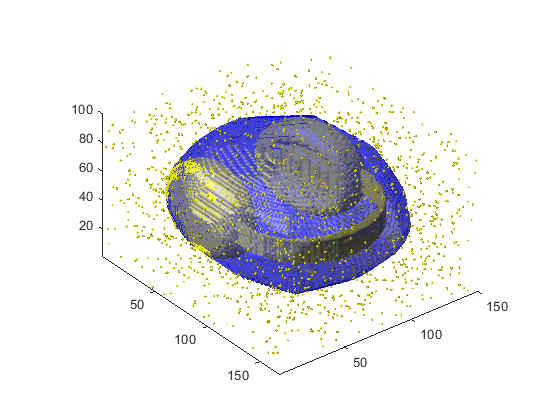}}
    \caption{Approximate convex hulls by the convex layer method.}
    \label{fig:convex_layers}
\end{figure}

\begin{algorithm}[ht]
\caption{convex layer method for convex hull approximation}
\begin{algorithmic}[1]
    \State Input a set of points $X$ and the estimated number of outliers $k$, where $k$ should be smaller than $|X|$.
    \State Set $count=0$.
    \While{$count<k$}
    \State Compute the convex hull of $X$.
    \State Set $C_X$ be the set of points lying on the boundary of Conv$(X)$.
    \State Delete points in $C_X$ from $X$.
    \State $count=count+|C_X|$.
    \EndWhile
    \State Return Conv$(X)$ as the approximated convex hull.
\end{algorithmic}\label{algo:convex_layers}
\end{algorithm}

\begin{table}[ht]
\caption{The relative errors of our method and the convex layers method}
\centering
\begin{tabular}{@{}ccc@{}}
\toprule
method  & \textbf{camera}  & \textbf{headphone} \\ \midrule
our method & 8.45\%  & 9.80\% \\ 
convex layers method  & 33.29\% & 47.83\%   \\ \bottomrule
\end{tabular}
\label{table:error2}
\end{table}

\section{Conclusions and future works}\label{sec:cln}
In this work, we presented a novel level-set based method for convexity shape representation in any dimensions. The method uses the second order condition of convex functions to characterize the convexity. We also proved the equivalence between the convexity of object and the convexity of associated SDF. This new method is very simple and easy to implementation in real applications. Two applications with convexity priors in computer vision were discussed and an efficient algorithm for a general optimization problem with the proposed convexity prior constraint was presented. Experiments for 2 and 3 dimensional examples were conducted and presented to show the effectiveness of the proposed method. In the future research, we will devote our effort to exploring more representation methods of generic shape priors and more potential applications of our convexity prior, especially in high dimensional spaces.

%\FloatBarrier
\bibliographystyle{siamplain}
\bibliography{reference}

\end{document}